\documentclass[11pt,letterpaper]{article}

\usepackage{microtype}
\usepackage{graphicx}
\usepackage[top=1in, bottom=1in, left=1in, right=1in]{geometry}

\usepackage{hyperref}

\usepackage{enumitem}
\usepackage{booktabs}       
\usepackage{nicefrac}       
\usepackage{microtype}      
\usepackage[justification=centering,caption=false]{subfig}
\usepackage{arydshln}

\usepackage{amssymb} 
\usepackage{amsthm} 
\usepackage{amsmath} 
\usepackage{nicefrac}

\usepackage{cleveref}
\usepackage{mathtools}
\usepackage{algorithmic}
\usepackage{algorithm}
\usepackage{bbm}
\usepackage{multirow}

\usepackage{pifont}
\newcommand{\xmark}{\ding{55}}

\newtheorem{theorem}{Theorem}[section]
\newtheorem{lemma}[theorem]{Lemma}
\newtheorem{definition}[theorem]{Definition}

\newtheorem{claim}[theorem]{Claim}

\newtheorem*{theorem*}{Theorem}
\newtheorem{remark}[theorem]{Remark}

\newcommand{\conv}{\mathrm{conv}}
\newcommand{\R}{\mathbb{R}}

\DeclarePairedDelimiter{\abs}{\lvert}{\rvert}
\DeclarePairedDelimiter{\norm}{\lVert}{\rVert}
\DeclarePairedDelimiter{\inner}{\langle}{\rangle}

\DeclarePairedDelimiter{\set}{\{}{\}}
\DeclarePairedDelimiterX{\Set}[2]\{\}{%
  \, #1 \;\delimsize\vert\; #2 \,
}
\newcommand{\eqlabel}[1]{\label{#1}\tag{#1}}

\newcommand{\E}[2][]{\mathbb{E}_{#1}\left[#2\right]}

\newcommand\blfootnote[1]{%
  \begingroup
  \renewcommand\thefootnote{}\footnote{#1}%
  \addtocounter{footnote}{-1}%
  \endgroup
}

\title{\bf Data preprocessing to mitigate bias: \\ A maximum entropy based approach\blfootnote{This is the full version of a paper in ICML 2020.}}

\usepackage{authblk}

\author[1]{L. Elisa Celis} 
\author[1]{Vijay Keswani}
\author[1]{Nisheeth K. Vishnoi}
\affil[1]{Yale University}

\date{}

\begin{document}

\maketitle
\begin{abstract}
Data containing human or social attributes may over- or under-represent groups with respect to salient social attributes such as gender or race, which can lead to biases in downstream applications.
This paper presents an algorithmic framework that can be used as a data preprocessing method towards mitigating such bias.
 Unlike prior work, it can efficiently learn distributions over large domains, controllably adjust the representation rates of protected groups and achieve target fairness metrics such as statistical parity, yet remains close to the empirical distribution induced by the given dataset.
Our approach leverages the principle of maximum entropy -- amongst all distributions satisfying a given set of constraints, we should choose the one closest in KL-divergence to a given prior.
While maximum entropy distributions can succinctly encode distributions over large domains, they can be difficult to compute. 
Our main  contribution is an instantiation of this framework for our set of constraints and priors, which encode our bias mitigation goals, and that runs in time polynomial in the {\em dimension} of the data.
Empirically, we observe that samples from the learned distribution have desired representation rates and statistical rates, and when used for training a classifier incurs only a slight loss in accuracy while maintaining fairness properties.
\end{abstract}

\thispagestyle{empty}

\newpage
\thispagestyle{empty}

\tableofcontents
\newpage

\setcounter{page}{1}

\section{Introduction}

Datasets often under- or over-represent social groups defined by salient attributes such as gender and race, and can be a significant source of bias leading to discrimination in the machine learning applications that use this data \cite{ON2016,Calders2013,KayMM2015}. 
Methods to debias data strive to ensure that either
1) the representation of salient social groups in the data is consistent with ground truth \cite{king2001logistic,chawla2002smote,zelaya2019towards}, or 2) the outcomes (where applicable) across salient social groups are fair  \cite{calders2009building,kamiran2012data,wang2019repairing,calmon2017optimized,xu2018fairgan,feldman2015certifying,del2018obtaining}.
The goal of this paper is to learn a distribution that corrects for \emph{representation} {and}  \emph{outcome} fairness but also remains as \emph{close} as possible to the original distribution from which the dataset was drawn.
Such a distribution allows us to generate new pseudo-data that can be used in downstream applications which is both true to the original dataset yet mitigates the biases it  contains; this has the additional benefit of not requiring the original data to be released when there are privacy concerns.
Learning this distribution in time polynomial in the size of the dataset and dimension of the domain (as opposed to the size of the domain, which is exponential in the number of attributes and class labels) is crucial in order for the method to be {scalable}.
Further, attaining provable guarantees on the efficiency and desired fairness properties is an important concern.
Hence, the question arises: \\
\emph{Can we  develop methods to learn accurate distributions that do not suffer from biases, can be computed efficiently over large domains, and come with theoretical guarantees?}

\paragraph{Our contributions.} 
We propose a framework based on the maximum entropy principle which asserts that among all distributions satisfying observed constraints one should choose the distribution that is ``maximally non-committal'' with regard to the missing information.
It has its origins in the works of Boltzmann, Gibbs and Jaynes \cite{gibbs1902elementary,Jaynes1,Jaynes2} and it  
is widely used in learning \cite{dudik2007maximum,SinghV14}.
Typically, it is used to learn probabilistic models of data from samples by finding the distribution over the domain that minimizes the KL-divergence with respect to a ``prior'' distribution, and whose expectation matches the empirical average obtained from the samples.  

Our framework leverages two properties of max-entropy distributions: 1)  any entropy maximizing  distribution can be succinctly represented with a small (proportional to the dimension of the data) number of parameters (a consequence of  duality) and, 2) the prior and expectation vector provides simple and interpretable ``knobs'' with which to control the statistical properties of the learned distribution.

We show that by appropriately setting the prior distribution and the expectation vector, we can provably enforce constraints on the fairness of the resulting max-entropy distribution, as measured by the representation rate (the ratio of the probability assigned to the under-represented group and the probability assigned to the over-represented group - Definition~\ref{defn:representation_rate}) and statistical rate (the ratio of the probability of belonging to a particular class given individual is in the under-represented group and the probability of belonging to the same class given individual is in the over-represented group - Definition~\ref{defn:statistical_rate}); see Theorem \ref{thm:statistical_rate_bound}. 
However, existing algorithms to compute max-entropy distributions depend on the existence of fast {oracles to evaluate the dual objective function and  bounds on the magnitude of the optimal (dual) parameters} \cite{SinghV14,straszakMED2019}.
Our main technical contribution addresses these problems by showing the existence of an efficient and scalable algorithm for gradient and Hessian oracles for our setting and a bound on the magnitude of the optimal parameters that is polynomial in the dimension. 
This leads to algorithms for computing the max-entropy distribution that runs in time polynomial in the size of the dataset and dimension of the domain (Theorem \ref{thm:computability}).
{Thus, our preprocessing framework for debiasing data comes with a provably fast algorithm.}

Empirically, we  evaluate the fairness and accuracy of the distributions generated by applying our framework to the Adult and COMPAS datasets, with gender as the protected attribute. 
Unlike prior work, the distributions obtained using the above parameters perform well for \emph{both} representational and outcome-dependent fairness metrics. 
We further show that classifiers trained on samples from our distributions achieve high fairness (as measured by the classifier's statistical rate) with minimal loss to accuracy.
Both with regard to the learned distributions and the classifiers trained on the de-biased data, our approach either matches or surpasses the performance of other state-of-the-art approaches across both fairness and accuracy metrics.
Further, it is efficient on datasets with large domains (e.g., approx $10^{11}$ for the large COMPAS dataset), for which some other approaches are infeasible with regard to runtime.

\begin{table*}
\caption{\textbf{Comparison of our paper with related work:} 
The first two rows denote the fairness metrics that can be controlled by each approach (see Definitions \ref{defn:representation_rate} and \ref{defn:statistical_rate}).
The last two rows denote whether the approach has the ability to sample from the entire domain, and whether it has a succinct representation.
We compare our performance against these methods empirically in Section~\ref{sec:main_experiments}.
}
\vspace{0.1in}
\centering
\begin{tabular}{lcccc}
 \toprule
 Properties & \cite{kamiran2012data}  & \cite{king2001logistic} & \cite{calmon2017optimized} &  {This paper} \\
 \midrule
 - Statistical Rate & \checkmark (only for $\tau=1$) & \xmark & \checkmark (only for $\tau=1$)  & \checkmark \\
 - Representation Rate & \xmark  &\checkmark  & \xmark & \checkmark \\
 \midrule
 - Entire domain & \xmark    & \xmark & \checkmark &   \checkmark\\
 - Succinct representation & \checkmark & \checkmark & \xmark & \checkmark\\
 \bottomrule
\end{tabular}
\label{tab:comparison}
\end{table*}

\paragraph{Related work.} 
Prior work on this problem falls, roughly, into two categories:  1) those that try to modify the dataset either by reassigning the protected attributes  or reweighting the existing datapoints \cite{calders2009building,kamiran2012data,wang2019repairing,king2001logistic},  or 2) those that try to  learn a distribution satisfying given constraints defined by the target fairness metric on the entire domain  \cite{calmon2017optimized}.

The first set of methods often leads to efficient algorithms, but are unable to generate points from the domain that are not in the given dataset; hence, the classifiers trained on the re-weighted dataset may not generalize well \cite{chawla2009data}.
Unlike the re-labeling/re-weighting approach of  \cite{calders2009building,kamiran2009classifying,kamiran2012data,king2001logistic} or the repair methods of \cite{del2018obtaining,wang2019repairing,feldman2015certifying,zemel2013learning}, we instead aim to learn a debiased version of the underlying distribution of the dataset across the entire domain.
The second approach also aims to learn a debiased distribution on the entire domain. E.g., \cite{calmon2017optimized} presents an optimization-based approach to learning a distribution that is close to the empirical distribution induced by the samples subject to fairness constraints.
However, as their optimization problem has a variable for each point in the domain, the running time of their algorithm is at least the size of the domain, which is exponential in the dimension of the data, and hence often infeasible for large datasets.
Since the max-entropy distribution can be efficiently represented using the  dual parameters, our framework does not suffer from the enumeration problem of  \cite{calders2009building} and the inefficiency for large domains as in \cite{calmon2017optimized}.
See Table~\ref{tab:comparison} for a summary of the properties of our framework with key related prior work. 
Other preprocessing methods include selecting a subset of data that satisfies specified fairness constraints such as representation rate  without attempting to model the distribution \cite{FatML,pmlr-v80-celis18a}.

GAN-based approaches towards mitigating bias \cite{mariani2018bagan,sattigeri2018fairness,xu2018fairgan} are inherently designed to simulate continuous distributions and are neither optimized for discrete domains that we consider in this paper nor are prevalently used for social data and benchmark datasets for fairness in ML.
While \cite{choiBMDSS17,xu2018fairgan} suggest methods to round the final samples to the discrete domain, it is not clear whether such rounding procedures preserve the distribution for larger domains.

While our framework is based on preprocessing the dataset, bias in downstream classification tasks can also be addressed by modifying the classifier itself.
Prior work in this direction fall into two categories: inprocessing methods that change the objective function optimized during training to include fairness constraints \cite{CHKV18,ZLM18}, and post-processing methods that modify the outcome of the existing machine learning models by changing the decision boundary \cite{KKZ12,HPS16}.

\section{Preliminaries}
\subsection{Dataset \& Domain}
We consider data from a discrete domain $\Omega:=\Omega_1 \times  \cdots \times \Omega_d = \{0,1\}^d$, i.e., each attribute $\Omega_i$ is binary.\footnote{Our results can be extended to domains with discrete or categorical attributes by encoding an attribute of size $k$ as binary using one-hot encodings: i.e., replace the cell with $e \in \{0,1\}^k$ where for a value $j \in [k]$ we set $e = \{e_1,\ldots,e_k\}$ with $e_j = 1$ and $e_\ell = 0$ for all $\ell \neq k$.
To handle continuous features, one can apply discretization to reduce a continuous feature to a non-binary discrete feature. 
However, there is a natural tradeoff between domain size and correctness.
We refer the reader to the survey \cite{kotsiantis2006discretization} for research on discretization techniques.}
The convex hull of $\Omega$ is denoted by $\conv(\Omega)=[0,1]^d$ and the size of the domain $\Omega$ is $2^d$, i.e., exponential in the dimension $d$.
We let the set (not multiset) $\mathcal{S} \subseteq \Omega$, along with a frequency $n_\alpha \geq 1$ for each point $\alpha \in \mathcal{S}$, denote a dataset consisting of $N = \sum_{\alpha \in \mathcal S} n_\alpha$ distinct points. 
We consider the attributes of $\Omega$, {indexed by the set $[d] := \set{1, \dots, d}$}, as partitioned into three index sets  where 1) $I_z$ denotes the indices of protected attributes, 2) $I_y$ denotes the set of outcomes or class labels considered for fairness metric evaluation, and 3) $I_x$ denotes the remaining attributes.
We denote the corresponding sub-domains by $\mathcal{X}:=\times_{i\in I_x} \Omega_i$, $\mathcal{Y}:=\times_{i\in I_y} \Omega_i$, and $\mathcal{Z}:=\times_{i\in I_z} \Omega_i$.

\subsection{Fairness metrics}
We consider the following two common fairness metrics; the first is ``representational'' (also known as ``outcome independent'') and depends only on the protected attributes and not on the class label, and the second one is an ``outcome dependent'' and depends on both the protected attribute and the class label.
\begin{definition}[\bf Representation rate] \label{defn:representation_rate} For $\tau  \in (0,1],$ a distribution $p : \Omega \rightarrow [0,1]$ is said to have  representation rate $\tau$ with respect to a protected attribute $\ell \in I_z$ if for all $ z_i, z_j \in \Omega_\ell$, we have
$$\;\frac{ p[Z = z_i]}{p[Z = z_j]} \geq \tau,$$
where $Z$ is distributed according to the marginal of  $p$  restricted to $\Omega_\ell$. 
\end{definition}
\begin{definition}[\bf Statistical rate] \label{defn:statistical_rate}
For $\tau  \in (0,1],$ a distribution $p : \Omega \rightarrow [0,1]$ is said to have statistical rate $\tau$ with respect to a protected attribute $\ell \in I_z$ and a class label $y \in \mathcal{Y}$ if  for all  $ z_i, z_j \in \Omega_\ell$, we have
$$\frac{ p[Y = y \mid Z = z_i]}{p[Y = y \mid Z = z_j]} \geq \tau,$$
where $Y$ is the random variable when $p$ is restricted to $\mathcal{Y}$ and $Z$  when $p$ is restricted to $\Omega_\ell$. 
\end{definition}
We also refer to the statistical rate when the outcome labels are instead obtained using a classifier $f : \mathcal{X} \times \mathcal{Z} \to \mathcal{Y}$. The classifier is said to have statistical rate $\tau$ if for all $ z_i, z_j \in \Omega_\ell$, we have
$$\frac{ \mathbb{P}[f(\alpha) = y \mid Z = z_i]}{\mathbb{P}[f(\alpha) = y \mid Z = z_j]} \geq \tau,$$ where the probability is over the empirical distribution of the test data.

In the definitions above,  $\tau=1$ can be thought of as ``{perfect}'' fairness and is referred to as representation parity and statistical parity respectively.
In practice, however, these perfect measures of fairness are often relaxed: a popular example is the ``80\% rule'' in US labor law \cite{biddle2006adverse} to address disparate impact in employment, which corresponds to $\tau=0.8$.
The exact value of $\tau$ desired is context-dependent and will vary by application and domain.

\subsection{The reweighting approach to debiasing data}
A weight $w(\alpha)$ is assigned to each data point $\alpha \in \mathcal{S}$ such that $w(\alpha) \geq 0$, and 
$\sum_{\alpha \in \mathcal{S}} w(\alpha)=1$. 
I.e., a probability distribution over  samples is computed.
These weights are carefully chosen in order to satisfy the desired fairness metrics, such as statistical parity \cite{kamiran2012data} or representation parity \cite{king2001logistic}.

\subsection{The optimization approach to debiasing data}
The goal of learning a debiased probability distribution over the entire domain is formulated as a constrained optimization problem over the space $\mathcal{P}$ of all probability distributions over  $\Omega$ (and not just $\mathcal{S}$).
A prior distribution $q$ is chosen that is usually supported on $\mathcal{S}$,
a distance measure $D$ is chosen to compare two probability distributions, 
and a function $J: \mathcal{P} \rightarrow \mathbb{R}^s$ that encodes the fairness criteria on the distribution is given.
The goal is to find the solution to the following optimization problem: 
$ \min_{p \in \mathcal{P}} D(p,q)  \  \mathrm{s.t.} \   J(p) = {0}. 
$
For instance, \cite{calmon2017optimized} use the total variation (TV) distance as the distance function and encode the fairness criteria as a linear constraint on the distribution.

\subsection{The maximum entropy framework}
Given  $\Omega\subseteq \R^d$, a {\em prior} distribution $q:\Omega\to[0,1]$ and a  {\em marginal} vector  $\theta\in\conv(\Omega)$, the maximum entropy distribution $p^\star:\Omega\to[0,1]$ is the maximizer of the following convex program,
\begin{align*}
 \sup_{p\in\R^{\abs{\Omega}}_{\geq0}}  &\sum_{\alpha\in\Omega} p(\alpha) \log\frac{q(\alpha)}{p(\alpha)}, \tag{primal-MaxEnt} \label{eq:primal-program} \\ 
\text{s.t.}
&\sum_{\alpha\in\Omega} \alpha p(\alpha) = \theta \ \ \mbox{ and }  \ \ \sum_{\alpha\in\Omega} p(\alpha) = 1. 
\end{align*}
\noindent The objective can be viewed as minimizing the KL-divergence with respect to the prior $q$. 
{To make this program well defined, if $q(\alpha)=0$, one has to restrict $p(\alpha)=0$ and define $\log \frac{0}{0}=1$.
The maximum entropy framework is traditionally used to learn a distribution over $\Omega$ by setting $\theta:=\frac{1}{N}\sum_{\alpha  \in \mathcal{S}} \alpha \cdot n_\alpha$ and $q$ to be the uniform distribution over $\Omega$.
This maximizes entropy while satisfying the constraint that the marginal is the same as the empirical marginal.
It is supported over the entire domain $\Omega$ (as $q$ is also supported on all of $\Omega$) and, as argued in the literature \cite{dudik2007maximum,SinghV14}, is information-theoretically the  ``least constraining'' choice on the distribution that can explain the statistics of $\mathcal{S}$.
Later we consider other choices for $q$ that take $\mathcal{S}$ and our fairness goals into account and are also supported over the entire domain $\Omega$.

Computationally, the number of variables in ~\eqref{eq:primal-program} is equal to the size of the domain and, hence does not seem scalable.
However, a key property of this optimization problem is that it suffices to solve the dual (see below) that only has $d$ variables (i.e., the dimension of the domain and not the size of the domain):
\begin{equation} \tag{dual-MaxEnt}
 \inf_{\lambda\in\R^d} h_{\theta,q}(\lambda):=\log \left(\sum_{\alpha\in\Omega} q(\alpha)e^{\inner{\alpha-\theta,\lambda}}\right),
\label{eq:dual-program}
\end{equation}
where the function $h_{\theta,q}:\R^d \to\R$ is referred to as the dual max-entropy objective.
For the objectives of the primal and dual to be equal (i.e., for strong duality to hold), one needs that $\theta$ lie in the ``relative interior'' of $\conv(\Omega)$; see \cite{SinghV14}. 
In the case $\conv(\Omega)=[0,1]^d$, this simply means that $0<\theta_i < 1$ for all $1 \leq i \leq d$.
This is satisfied if for each attribute $\Omega_i$ there is at least one point in the set $\mathcal{S}$ that takes value $0$ and at least one point that takes value $1$.

Strong duality also implies  that, if $\lambda^\star$ is a  minimizer of $h_{\theta,q}$, then $p^\star$ can be computed as 
\begin{equation*}
p^\star(\alpha) = \frac{q(\alpha) e^{\inner{\lambda^\star,\alpha}}}{\sum_{\beta\in\Omega} q(\beta) e^{\inner{\lambda^\star,\beta}}};
\end{equation*}
see \cite{dudik2007maximum,SinghV14}.
Thus, the distribution $p^\star$ can be represented only using $d$ numbers $\lambda_i^\star$ for $1 \leq i \leq d$.
However, note that as some $\theta_i$ go close to an integral value or some $q(\alpha)\to 0$, these optimal dual variables might tend to infinity. 
Further, given a  $\lambda$, computing $h_{\theta,q}$ requires computing a summation over the entire domain $\Omega$ -- even in the simplest setting when $q$ is the uniform distribution on $\Omega$ -- that can a priori take time proportional to $|\Omega|=2^d$. 
Hence, even though the dual optimization problem is convex and has a small number of variables ($d$), to obtain a polynomial (in $d$) time algorithm to solve it, we need
both  an algorithm that evaluate the dual function $h_{\theta,q}$ (a summation over the entire domain $\Omega$) and its gradient efficiently at a given point $\lambda$, and 
 (roughly) a bound on $\|\lambda^\star \|_2$ that is polynomial in $d$.

\section{Our framework}
Our approach for preprocessing data uses the maximum entropy framework and combines both the reweighting and optimization approaches.
Recall that  the maximum entropy framework requires the  specification of the marginal vector $\theta$ and a prior distribution $q$.
We use $q$ and $\theta$ to enforce our goals of controlling representation and statistical rates as defined in Definitions \ref{defn:representation_rate} and \ref{defn:statistical_rate}, while at the same time ensuring that the learned distribution has support all of $\Omega$ and is efficiently computable in the dimension of $\Omega$.
Another advantage of computing the max-entropy distribution (as opposed to simply {using the prior $q$}) is that it pushes the prior towards the empirical distribution of the raw dataset, while maintaining the fairness properties of the prior.
This leads to a distribution which is close to the empirical distribution and has fairness guarantees.

\subsection{Prior distributions} 
Let $u$ denote the uniform distribution on $\Omega$:
$u(\alpha):=\frac{1}{|\Omega|}$ for all $\alpha \in \Omega$.
Note that the uniform distribution satisfies statistical rate with $\tau=1$.
We also use a reweighting algorithm (Algorithm~\ref{algo:reweighting}) to compute a distribution $w$ supported on $\mathcal{S}$. 
Our algorithm is inspired by the work of \cite{kamiran2012data} and, for any given $\tau \in (0,1]$, Algorithm~\ref{algo:reweighting} can ensure that $w$ satisfies the $\tau$-statistical rate property; see Theorem~\ref{thm:reweighting}.
We introduce a parameter $C \in [0,1]$ that allows us to interpolate between  $w$ and $u$ and define:
\begin{equation}\label{eq:qwc}
 q^w_C:= C \cdot u + (1-C) \cdot w.
\end{equation}
A desirable property of $q_C^w$, that we show is true, is that the dual objective function  $h_{\theta,q_C^w}$ and its gradient are computable in time polynomial in $N, d$ and the number of bits needed to represent $\theta$ for any weight vector $w$ supported on $\mathcal{S}$; see Lemma~\ref{lem:counting-oracles}.
Further, we show that, if $w$ has $\tau$-statistical rate, then for any $C\in [0,1]$, the distribution $q^w_C$ also has $\tau$-statistical rate; see Theorem~\ref{thm:reweighting}.

Thus, the family of priors we consider present no computational bottleneck over exponential-sized domains.
Moreover, by choosing the parameter $C$, our framework allows the user to control how close they would like the learned distribution to be to the empirical distribution induced by $\mathcal{S}$. 
Finally, using appropriate weights $w$ which encode the desired statistical rate, one can aim to ensure that the optimal distribution to the max-entropy program is also close to satisfying statistical parity (Theorem~\ref{thm:statistical_rate_bound}). 

\subsection{Marginal vectors}
The simplest choice for the marginal vector $\theta$ is the marginal of the empirical distribution $\frac{1}{N} \sum_{\alpha \in \mathcal{S}} n_\alpha \cdot \alpha$.
However, in our framework, the user can select any vector   $\theta$.
In particular, to control the representation rate of the learned distribution with respect to a protected attribute $\ell$, we can choose to set it differently.
For instance, if $\Omega_\ell=\{0,1\}$ and we would like that in learned distribution the probability of this attribute being $1$ is $0.5$, it suffices to set $\theta_\ell=0.5$.
This follows immediately from the constraint imposed in the max-entropy framework. 
Once we fix a choice of $\theta$ and $q$, we need to solve the dual of the max-entropy program and we discuss this in the next section. 
The dual optimal $\lambda^\star$ can then be used to sample from the distribution $p^\star$ in a standard manner; see Appendix~\ref{sec:sampling_oracle}.

\setlength{\textfloatsep}{10pt}
\begin{algorithm}[t]
   \caption{Re-weighting algorithm to assign weights to samples for the prior distribution}
   \label{algo:reweighting}
\begin{algorithmic}[1]
   \STATE {\bfseries Input:} Dataset $\mathcal{S}:= \set{(X_\alpha, Y_\alpha, Z_\alpha)}_{_\alpha\in \mathcal{S}}\subseteq \mathcal{X}\times\mathcal{Y}\times \Omega_\ell$, frequency list $\set{n_\alpha}_{\alpha \in \mathcal{S}}$ and parameter $\tau \in (0,1]$
   \FOR{$y\in\mathcal{Y}$}
   	\STATE $c(y)\gets \sum_{\alpha \in \mathcal{S}} \mathbf{1}({Y_\alpha =y}) \cdot n_\alpha$ 
   	\STATE $c(y,0)\gets \frac{1}{\tau} \cdot \sum_{\alpha \in \mathcal{S}} \mathbf{1}({Y_\alpha = y, Z_\alpha=0})  \cdot n_\alpha$ 
   	\STATE $c(y,1)\gets \sum_{\alpha \in \mathcal{S}} \mathbf{1}({Y_\alpha = y, Z_\alpha=1})  \cdot n_\alpha$ 
   \ENDFOR
   \STATE $w \gets \textbf{0}$
   \FOR {$\alpha \in \mathcal{S}$}
   	\STATE $w(\alpha) \gets n_\alpha \cdot \nicefrac{c(Y_\alpha)}{c(Y_\alpha, Z_\alpha)}$
   \ENDFOR
   
   \STATE $W \gets \sum_{\alpha \in S} w(\alpha)$
   \STATE {\bf return} $\set{\nicefrac{w(\alpha)}{W}}_{\alpha \in S}$
\end{algorithmic}
\end{algorithm}

\section{Theoretical results}
Throughout this section we assume that we are given $C \in [0,1]$,   $\mathcal{S} \subseteq \Omega$ and the frequency of elements in $\mathcal{S}$, $\set{n_\alpha}_{\alpha \in \mathcal{S}}$.

\subsection{The reweighting algorithm and its properties}
We start by showing that there is an efficient algorithm to compute the weights $w$ discussed in the previous section.
\begin{theorem}[\bf Guarantees on the reweighting algorithm]
Given the dataset $\mathcal{S}$, frequencies $\set{n_\alpha}_{\alpha \in \mathcal{S}}$ and a $\tau \in [0,1]$, Algorithm~\ref{algo:reweighting} outputs a probability distribution $w:\mathcal{S} \to [0,1]$ such that 
\begin{enumerate}
    \item The algorithm runs in time linear in $N$.
    \item $q_C^w$, defined in Eq. \eqref{eq:qwc} using $w$, satisfies $\tau$-statistical rate, i.e, for any $y \in \mathcal{Y}$ and for all $z_1, z_2 \in \Omega_\ell$,
    \begin{equation*}
    \frac{q_C^w(Y = y \mid Z= z_1)}{q_C^w(Y = y \mid Z= z_2)} \geq \tau.
    \end{equation*}
\end{enumerate}
\label{thm:reweighting}
\end{theorem}

\noindent
The proof of this theorem uses the fact that $q_C^w$ is a convex combination of uniform distribution, which has statistical rate 1, and weights from Algorithm~\ref{algo:reweighting}, which by construction satisfy statistical rate $\tau$; it is presented in Section~\ref{sec:reweighting_appendix}.

\subsection{Computability of maximum entropy distributions}
Since the prior distribution $q_C^w$ is not uniform in general, the optimal distribution $p^\star$ is not a product distribution.
Thus, as noted earlier, the number of variables in (primal-MaxEnt) is $|\Omega|=2^d$, i.e., exponential in $d$,
and standard methods from convex programming to directly solve \ref{eq:primal-program} do not lead to efficient algorithms. 
Instead, we focus on  computing (dual-MaxEnt).
Towards this, we appeal to the general algorithmic framework of  \cite{SinghV14,straszakMED2019}.
To use their framework, we need to provide (1) a bound on  $\|\lambda^\star\|_2$ and (2) an efficient algorithm (polynomial in $d$) to  evaluate the dual objective $h_{\theta,q}$ and its gradient.  
Towards (1), we prove the following.
\begin{lemma}[\bf Bound on the optimal dual solution] \label{lem:bounding_box}
Suppose $\theta$ is such that there is an $\eta>0$ for which we have 
$\eta < \theta_i < 1-\eta$ for all $i \in [d]$. Then, the optimal dual solution corresponding to such a $\theta$ and $q_C^w$ satisfies
\begin{equation*} 
\|\lambda^\star \|_2 \leq \frac{d}{\eta}\log \frac{1}{C}. 
\end{equation*}
\end{lemma}
\noindent
The proof uses a result from \cite{SinghV14} and is provided in Section~\ref{sec:proof_bounding_box}.
We note that, for our applications, we can show that the assumption on $\theta$ follows from an assumption on the ``non-redundancy'' of the data set. 
Using recent results of \cite{straszakMED2019}, we can get around this assumption and we omit the details from this version of the paper.

Towards (2), we show that  $q_C^w$ has  the property that not only  can one evaluate $h_{\theta,q_C^w}$, but also its gradient (and Hessian).

\begin{lemma}[\bf Oracles for the dual objective function]
There is an algorithm that, given a reweighted distribution $w:\mathcal{S} \to (0,1]$, values $\theta,\lambda \in \R^d$, and distribution $q=q_C^w$, computes $h_{\theta,q}(\lambda),$ $\nabla h_{\theta,q}(\lambda),$ and $\nabla^2 h_{\theta,q}(\lambda)$ in time polynomial in $N, d$ and the bit complexities of all the numbers involved: $w(\alpha)$ for $\alpha \in \mathcal{S}$,  and $e^{\lambda_i}, \theta_i$ for $1 \leq i \leq d.$
\label{lem:counting-oracles}
\end{lemma}

\noindent
The proof of this lemma is provided in Section~\ref{sec:proof_counting_oracle}} and the complete algorithm is given in Appendix~\ref{sec:full_algorithm}.
It uses the fact that $q_C^w$ is a convex combination of uniform distribution (for which efficient oracles can be constructed) and a weighted distribution supported only on $\mathcal{S}$, and can be generalized to any prior $q$ that similarly satisfies these properties.

Thus, as a direct corollary to Theorem 2.8 in the arxiv version of \cite{SinghV14} we obtain the following. 
\begin{theorem}[\bf Efficient algorithm for  max-entropy distributions] \label{thm:computability}
There is an algorithm that, given a reweighted distribution $w:\mathcal{S} \to [0,1]$, a $\theta \in [\eta,1-\eta]^d$, and an $\varepsilon>0$, computes a $\lambda^\circ$ such that
\begin{equation*}
     h_{\theta,q}(\lambda^\circ) \leq h_{\theta,q}(\lambda^\star) + \varepsilon.
\end{equation*}
Here $\lambda^\star$ 
is an optimal solution to the dual of the max-entropy convex program for $q:=q_C^w$ and $\theta$. 
The running time of the
algorithm is polynomial in $d,\frac{1}{\eta},\frac{1}{\varepsilon}$  and the number of bits needed to represent $\theta$ and $w$.
\end{theorem}

\subsection{Fairness guarantees}
Given a marginal vector $\theta$ that has representation rate $\tau$, we can bound the statistical rate and representation rate of the the max-entropy distribution obtained using $q_C^w$ and $\theta$.

\begin{theorem}[\textbf{Fairness guarantees}] \label{thm:statistical_rate_bound}
Given the dataset $\mathcal{S}$, protected attribute $\ell \in I_z$, class label $y \in \mathcal{Y}$ and parameters $\tau, C \in [0,1]$, let $w : S \to [0,1]$ be the reweighted distribution obtained from Algorithm~\ref{algo:reweighting}. 
Suppose $\theta$ is a vector that satisfies $\frac{1}{2} \leq \theta_{\ell} \leq \frac{1}{1+\tau}$.
The max-entropy distribution $p^\star$ corresponding to the prior distribution $q_C^w$ and expected value $\theta$ has statistical rate at least $\tau'$ with respect to $\ell$ and $y$, where
$$\tau' = \tau -\frac{  4\delta \cdot (1 + \tau)}{C + 4\delta},$$ and 
$\delta = \max_{z \in \Omega_\ell} \abs{ p^\star(Y=y, Z=z) - q_C^w(Y=y, Z=z)}$;
here $Y$ is the random variable when the distribution is restricted to $\mathcal{Y}$ and $Z$ is the random variable when the distribution is restricted to $\Omega_\ell$.
\end{theorem}

\noindent
The condition on $\theta$, when simplified, implies that
$\nicefrac{(1 - \theta_\ell)}{\theta_\ell} \geq \tau$ and $\nicefrac{\theta_\ell}{(1 - \theta_\ell)} \geq 1$, i.e., the marginal probability of $Z=0$ is atleast $\tau$ times the marginal probability of $Z=1$. 
This directly implies that the representation rate of $p^\star$ is at least $\tau$.
As we control the statistical rate using the prior $q_C^w$, the statistical rate of $p^\star$ depends on the distance between $q_C^w$ and $p^\star$.
The proof of Theorem~\ref{thm:statistical_rate_bound} is provided in  Section~\ref{sec:proof_stat_rate}.
 
\begin{remark} \label{rem:other_theta}
Two natural choices for $\theta$ that satisfy the conditions of Theorem~\ref{thm:statistical_rate_bound} are the following: %
\begin{enumerate}
\item  The reweighted vector 
$\theta^w := \sum_{\alpha \in \mathcal{S}} w(\alpha)\cdot \alpha$, where $w$ is the weight distribution obtained using Algorithm~\ref{algo:reweighting}; since $w$ has representation rate $\tau$, it can be seen that $\theta^w_\ell = 1/(1+\tau)$. 

\item The vector $\theta^b$ that is the mean of the dataset $\mathcal{S}$ for all non-protected attributes and class labels, and is balanced across the values of any protected attribute.
I.e., 
\begin{equation*}
\theta^{b} := \left(\sum_{\alpha \in \mathcal{S}} \frac{n_\alpha}{N} X_\alpha, \sum_{\alpha \in \mathcal{S}}\frac{n_\alpha}{N} Y_\alpha, \frac{1}{2}\right).
\end{equation*}
\end{enumerate}
\end{remark}

\section{Empirical analysis} \label{sec:main_experiments}
Our approach, as described above, is flexible and can be used for a variety of applications. 
\footnote{The code for our framework is available at \url{https://github.com/vijaykeswani/Fair-Max-Entropy-Distributions}.}
In this section we show its efficacy as compared with other state-of-the-art data debiasing approaches, in particular reweighting methods by \cite{kamiran2012data, king2001logistic} and an optimization method by \cite{calmon2017optimized}. 
We consider two applications and three different domain sizes: The \textbf{COMPAS} criminal defense dataset using two versions of the data with differently sized domains, and the \textbf{Adult} financial dataset. 
With regard to fairness, we compare the statistical rate and representation rate of the de-biased datasets as well as the statistical rate of a classifier trained on the de-biased data.
With regard to accuracy, we report both the divergence of the de-biased dataset from the raw data, as well as the resulting classifier accuracy.
We find that our methods perform at least as well as if not better than existing approaches across all fairness metrics; in particular, ours are the only approaches that can attain a good representation rate while, simultaneously, attaining good statistical rate both with regard to the data and the classifier.
Further, the loss as compared to the classifier accuracy when trained on raw data is minimal, even when the KL divergence between our distribution and the empirical distribution is large as compared to other methods.
Finally, we report the runtime of finding the de-biased distributions, and find that our method scales well even for large domains of size $\sim 10^{11}$.

\subsection{Setup for empirical analysis} \label{sec:setup}

\noindent
{\bf Datasets.}
We consider two benchmark datasets from the fairness in machine learning literature.\footnote{The details of both datasets, including a description of features are presented in Appendix~\ref{sec:experiments1} and \ref{sec:other_experiments}.} 

(a) The {\textbf{COMPAS} dataset \cite{compas,larson2016we}}   
contains information on criminal defendants at the time of trial (including criminal history, age, sex, and race), along with post-trail instances of recidivism (coded as any kind of re-arrest).
We use two versions of this dataset:
the \textbf{small} version has a domain of size $144$, and contains sex, race, age, priors count, and charge degree as features, and uses a binary marker of recidivism within two years as the label.
We separately consider race (preprocessed as binary with values ``Caucasian'' vs ``Not-Caucasian'') and gender (which is coded as binary) as  protected attributes.
The  \textbf{large} dataset has a domain of size approximately $1.4 \times 10^{11}$ and consists of 19 attributes, 6 different racial categories and additional features such as  the type of prior and juvenile prior counts.

(b) The {\textbf{Adult} dataset \cite{adult}} contains demographic information of individuals along with a binary label of whether their annual income is greater than \$50k, and has a domain of size $504$.
The demographic attributes include race, sex, age and years of education.
We take gender (which is coded as binary) as the protected attribute.

\noindent
{\bf Using our approach.}
We consider the prior distribution $q_C^{w}$, which assigns weights returned by Algorithm~\ref{algo:reweighting} for input $\mathcal{S}$ and $\tau=1$ and $C=0.5$.\footnote{This choice for $C$ is arbitrary; we evaluate performance as a function of $C$ in Appendix \ref{sec:experiments1}.}
Further, we consider the two different choices for the expectation vector as defined in Remark~\ref{rem:other_theta}, namely: 
(1) The weighted mean of the samples $\theta^{w}$ using the weights $w$
as obtained from Algorithm~\ref{algo:reweighting}, and 
(2) the empirical expectation vector with the marginal of the protected attribute modified to ensure equal representation of both groups $\theta^b$. In this case, since the protected attribute is binary we set $\theta^b_\ell = 1/2$.
\footnote{In Appendix~\ref{sec:experiments1} we evaluate the performance using alternate priors and expectation vectors such as  $q_C^{d}$ and $\theta_{\textrm{d}}$ which correspond to the raw data.
}

\noindent
{\bf Baselines and metrics. }
We compare against the raw data, simply taking the prior $q_C^{w}$ defined above, a reweighting method~\cite{kamiran2012data} for statistical parity, a reweighting method~\cite{king2001logistic} for representation parity, and an optimized preprocessing method \cite{calmon2017optimized}. 
We consider the distributions themselves in addition to classifiers trained on simulated datasets drawn from these distributions, and 
evaluate them with respect to well-studied metrics of fairness and accuracy. 

For fairness metrics, we report the statistical rate (see Definition~\ref{defn:statistical_rate}).
%
Note that this can be evaluated both with regard to the instantiation of the outcome variable in the simulated data, and with regard to the outcome predicted by the classifier; we report both.
We also report the representation rate (see Definition~\ref{defn:representation_rate}) of the simulated data; for gender this corresponds to the ratio between fraction of women and men in the simulated datasets, while for race this corresponds to the ratio between fraction of Caucasian and Non-Caucasian individuals in the simulated datasets.
For all fairness metrics, larger values, closer to 1, are considered to be ``more fair''.

We report the classifier accuracy when trained on the synthetic data. 
Further, we aim to capture the distance between the de-biased distribution and the distribution induced by the empirical samples. For the Adult dataset and small COMPAS dataset we report the KL-divergence.\footnote{For this to be well-defined, if a point does not appear in the dataset, before calculating KL-divergence, we assign it a very small non-zero probability ($\sim 10^{-7}$). } 
For the large COMPAS dataset, the KL-divergence is not appropriate as most of the domain is not represented in the data. 
We instead consider the covariance matrix of the output dataset  and the raw dataset and report the Frobenius norm of the difference of these matrices. 
In either case, lower values suggest the synthetic data better resembles the original dataset.
Lastly, we report the runtime (in seconds) of each approach.

\paragraph{Implementation details.}  
We perform 5-fold cross-validation for every dataset, i.e., we divide each dataset into five partitions. 
First, we select and combine four partitions into a training dataset and use this dataset to construct the distributions. 
Then we sample 10,000 elements from each distribution and train the classifier on this simulated dataset.
We then evaluate our metrics on this simulated dataset and classifier (where the classifier accuracy and statistical rate is measured over the test set, i.e., the fifth partition of the original dataset).
This sampling process is repeated 100 times for each distribution.
We repeat this process 5 times for each dataset, once for each fold. 
We report the mean across all (500) repetitions and folds.
Within each fold, the standard error across repetitions is low, less than 0.01 for all datasets and methods.
Hence, for each fold, we compute the mean of metrics across the 100 repetitions and then report the standard deviation of this quantity across folds.

We use a decision tree classifier with gini information criterion as the splitting rule. A Gaussian naive Bayes classifier gives similar results. Further details are presented in Appendix \ref{sec:experiments1}.
In the computation of the max-entropy distribution, we use a second-order algorithm inspired from works of \cite{ALOW17,CMTV17} that is also provably polynomial time in the parameters above and turns out to be slightly faster in practice. We present the details in Appendix~\ref{sec:full_algorithm}.
The machine specifications are a  1.8Ghz Intel Core i5 processor with 8GB memory.

\subsection{Empirical results}
The empirical results comparing our max-entropy approach against the state-of-the-art are reported in Table~\ref{tab:Clf_DI_Utlity_Compare} and graphically presented in Figure~\ref{fig:main_results}.
The performance of using just the prior $q_C^w$ is also reported in the table and the figure.
For all datasets, the statistical rate of max-entropy distributions is at least $0.97$, which is higher than that of the raw data and higher or comparable to other approaches, including those specifically designed to  optimize statistical parity \cite{calmon2017optimized, kamiran2012data}.
Additionally, the representation rate of max-entropy distributions is at least $0.97$, which is higher than that of the raw data and higher or similar to other approaches, including those specifically designed to optimize the representation rate \cite{king2001logistic}. 
Recall that both fairness metrics can be at most $1$; this suggests 
the synthetic data our distributions produce have a near-equal fraction of individuals from both groups of protected attribute values (women/men or Caucasian/Not-Caucasian) \emph{and} the probability of observing a favorable outcome is almost equally likely for individuals from both groups.

Note that  Theorem~\ref{thm:statistical_rate_bound} gives a bound on the statistical rate $\tau'$.
While this bound can be strong, the statistical rates we observe empirically are even better.
E.g., for the small COMPAS dataset with gender as the protected attribute, by plugging in the value of $\delta$ for prior $q_C^w$ and expected vector $\theta^w$, we get that $\tau' = 0.85$ (i.e., satisfying the 80\% rule), but we observe that empirically it is even higher (0.98).
However, the bound may not always be strong. 
E.g., or the Adult dataset, we only get $\tau' = 0.23$. 
In this case, the distance between the prior $q_C^w$ and max-entropy distribution $p^\star$ is large hence the bound on the statistical rate of $p^\star$, derived using $q_C^w$, is less accurate.
Still, the statistical rate of max-entropy distribution is observed to be $0.97$, suggesting that perhaps stronger fairness guarantees can be derived.

The statistical rate of  the classifiers trained on the synthetic data generated by our max-entropy approach is comparable or better than that from other methods, and significantly better than the statistical rate of the classifier trained on the raw data.
Hence, as desired, our approach leads to improved fairness in downstream applications.
This is despite the fact that the KL-divergence of the max-entropy distributions from the empirical distribution on the dataset is high compared to most other approaches. 
Still, we note that the difference between the max-entropy distributions and the empirical distribution tends to be smaller than the difference between the prior $q_C^w$ and the empirical distribution (as measured by KL divergence and the covariance matrix difference as discussed above). 
This suggests that, as expected, the max-entropy optimization helps push the re-weighted distribution towards the empirical distribution and
highlights the benefit of using a hybrid approach of reweighting and optimization.

For the COMPAS datasets, the raw data has the highest accuracy and the average loss in accuracy when using the datasets generated from max-entropy distributions is at most 0.03. 
This is comparable to the loss in accuracy when using datasets from other baseline algorithms.
In fact, for the small version of COMPAS dataset, the accuracy of the classifier trained on datasets from the max-entropy distribution using marginal $\theta^b$ is statistically similar to the accuracy of the classifier trained on the raw dataset.
For the Adult dataset, \cite{king2001logistic} achieves the same classifier accuracy as the raw dataset.
As the Adult dataset is relatively more gender-balanced than COMPAS datasets and outcomes are not considered, \cite{king2001logistic} do not need to modify the dataset significantly to achieve a high representation rate (indeed its KL-divergence from the empirical distribution of the raw data is the smallest).
In comparison, all other methods that aim to satisfy statistical parity (max-entropy approach, \cite{calmon2017optimized, kamiran2012data}) suffer a similar (but minimal) loss in accuracy of at most $0.03$.

With respect to runtime,
since \cite{kamiran2012data}, \cite{king2001logistic} and prior $q_C^w$ are simple re-weighting approaches and do not look at features other than class labels and protected attribute, it is not surprising that they have the best processing time.
Amongst the generative models, the max-entropy optimization using our algorithm is significantly faster than the optimization framework of \cite{calmon2017optimized}.
In fact, the algorithm of \cite{calmon2017optimized} is infeasible for larger domains, such as the large COMPAS dataset, and hence we are not able present the results of their algorithm on that dataset.

\begin{table*}[!hbtp]
\scriptsize
\caption{
\textbf{Empirical results.} Our max-entropy distributions use prior $q_C^{w}$ for $C = 0.5$ and expected value $\theta^{w} $  or $\theta^{b}$ (as defined in Remark~\ref{rem:other_theta}).
``SR'' denotes statistical rate, ``RR'' denotes representation rate, and ``Clf'' denotes classifier.
We report the mean across all folds and repetitions, with the standard deviation across folds in parentheses.
For each measurement and dataset, the results that are not statistically distinguishable at $\textrm{p-value} = 0.05$ from the best result across all baselines and approaches are given in bold.
Note that the approach is infeasible for larger domains, such as the large version of COMPAS datasets, and hence we do not present the results of \cite{calmon2017optimized} on that dataset.
{The results in this table are graphically presented in Figure~\ref{fig:main_results}.}
}
\vspace{0.1in}
\centering
\begin{tabular}{ p{0.1cm}p{0.1cm}p{0.1cm}p{1.6cm} p{1.5cm}  p{1.5cm} p{1.5cm} p{1.5cm} p{1.5cm} p{1.5cm} p{1.5cm} }
 \toprule
 \multicolumn{5}{c}{ }   &  \multicolumn{3}{c}{This paper} &  \multicolumn{3}{c}{Baselines} \\
\cmidrule{6-11}
 \multicolumn{4}{c}{}  &  Raw Data &  Prior $q_C^{w}$ & Max-Entropy with $q_C^{w}$, $\theta^{w} $ &  Max-Entropy with $q_C^{w}$, $\theta^{b} $ & \cite{calmon2017optimized}  &  \cite{kamiran2012data} &  \cite{king2001logistic} \\
\midrule 
\parbox[t]{4mm}{\multirow{5}{*}[-0.1in]{\rotatebox[origin=c]{90}{Adult }}}& \parbox[t]{4mm}{\multirow{5}{*}[-0.1in]{\rotatebox[origin=c]{90}{gender}}}& \parbox[t]{2mm}{\multirow{3}{*}[-0.02in]{\rotatebox[origin=c]{90}{\tiny{Fairness}}}}& Data SR& 0.36 (0) & \textbf{0.97} (0.02) & \textbf{0.98} (0.02) & \textbf{0.98} (0.02) & 0.96 (0.01) & \textbf{0.97} (0.02)  & 0.36 (0)\\ 
& & & Data RR  & 0.49 (0) &\textbf{0.97} (0.01) & \textbf{0.97} (0.02) & \textbf{0.99} (0.01) & 0.49 (0.01) & 0.49 (0.01) & \textbf{0.98} (0) \\ 
& & & Clf SR& 0.36 (0) & \textbf{0.96} (0.03) & \textbf{0.95} (0.02) & \textbf{0.96} (0.01) & \textbf{0.97} (0.01) &  0.85 (0.03)   & 0.36 (0) \\ 
\cmidrule{3-11}
& & \parbox[t]{2mm}{\multirow{2}{*}[-0.01in]{\rotatebox[origin=c]{90}{\tiny{Accuracy}}}}& KL-div w.r.t raw data & \textbf{0} (0) & 1.23 (0.03) & 0.24 (0.01) & 0.24 (0.01) & 0.16 (0) & 0.22 (0.01)   & 0.08 (0) \\ 
&& & Clf Acc & \textbf{0.80} (0) & {0.75} (0.01) & {0.77} (0.02) & 0.76 (0.01) & {0.77} (0.01) & {0.78} (0.01)   & \textbf{0.80} (0) \\ 
\cmidrule{3-11}
&& & Runtime & - & 0.73s & 10s & 10s & 62s  & 0.16s  & 0.57s \\ 
\midrule 
\parbox[t]{4mm}{\multirow{10}{*}[-0.40in]{\rotatebox[origin=c]{90}{COMPAS (small)}}}& \parbox[t]{4mm}{\multirow{5}{*}[-0.1in]{\rotatebox[origin=c]{90}{gender}}}&  \parbox[t]{2mm}{\multirow{3}{*}[-0.01in]{\rotatebox[origin=c]{90}{\tiny{Fairness}}}}& Data SR& 0.73 (0.02) & \textbf{0.98} (0.01) & \textbf{0.98} (0.02) & \textbf{0.99} (0.01) & 0.87 (0.02) & \textbf{0.98} (0.02)  & 0.73 (0.03) \\ 
& & & Data RR & 0.24 (0.01) & \textbf{0.97} (0.02) & \textbf{0.98} (0.01) & \textbf{0.98} (0.02) & 0.24 (0.01) & 0.24 (0.01)  & \textbf{0.98} (0) \\ 
& & & Clf SR& 0.72 (0.01) & \textbf{0.96} (0.02) & \textbf{0.95} (0.02) & \textbf{0.96} (0.02) & \textbf{0.93} (0.04) & \textbf{0.93} (0.03)  & 0.72 (0.01) \\ 
\cmidrule{3-11}
& & \parbox[t]{2mm}{\multirow{2}{*}[-0.01in]{\rotatebox[origin=c]{90}{\tiny{Accuracy}}}}& KL-div w.r.t raw data & \textbf{0} (0) & 0.57 (0.03) & 0.35 (0.01) & 0.37 (0.02) & {0.02} (0) & 0.14 (0.02)  & 0.24 (0)\\ 
&& & Clf Acc & \textbf{0.66} (0.01) & \textbf{0.65 }(0.01) & 0.64 (0.01) & \textbf{0.65} (0.02) & \textbf{0.66} (0.01) & \textbf{0.66} (0.01)  & \textbf{0.66} (0.01)\\ 
\cmidrule{3-11}
& & & Runtime & - & 0.06s & 2.5s & 2.6s & 25s  & 0.04s  & 0.10s \\ 
\cmidrule{2-11}
& \parbox[t]{4mm}{\multirow{5}{*}[-0.1in]{\rotatebox[origin=c]{90}{race}}}&  \parbox[t]{2mm}{\multirow{3}{*}[-0.01in]{\rotatebox[origin=c]{90}{\tiny{Fairness}}}}& Data SR& 0.76 (0.01) & \textbf{0.98} (0.01) & \textbf{0.98} (0.01) & \textbf{0.99} (0.01) & 0.93 (0.01) & \textbf{0.98} (0.01)  & 0.76 (0.01) \\ 
& & & Data RR & 0.66 (0.01) & \textbf{0.99} (0.01) & \textbf{0.99} (0.01) & \textbf{0.99} (0.01) & 0.74 (0.02) & 0.67 (0.02)  & \textbf{0.99} (0) \\ 
& & & Clf SR& 0.75 (0.02) & \textbf{0.95} (0.03) & \textbf{0.96} (0.01) & \textbf{0.94} (0.03) & {0.85} (0.09) & \textbf{0.96} (0.03)  & 0.75 (0.02) \\ 
\cmidrule{3-11}
& & \parbox[t]{2mm}{\multirow{2}{*}[-0.01in]{\rotatebox[origin=c]{90}{\tiny{Accuracy}}}}& KL-div w.r.t raw data & \textbf{0} (0) & 0.36 (0.02) & 0.13 (0.01) & {0.13} (0.01) & 0.02 (0.01) & 0.02 (0)  & 0.03 (0)\\ 
& & & Clf Acc & \textbf{0.66} (0.01) & \textbf{0.64} (0.02) & \textbf{0.65} (0.02) & \textbf{0.65} (0.01) & 0.58 (0.02) & \textbf{0.65} (0.01)  & \textbf{0.66} (0.01)\\ 
\cmidrule{3-11}
& & & Runtime & - & 0.06s & 2.5s & 2.6s & 25s  & 0.04s  & 0.10s \\ 
\midrule
\parbox[t]{4mm}{\multirow{10}{*}[-0.40in]{\rotatebox[origin=c]{90}{COMPAS (large)}}}& \parbox[t]{4mm}{\multirow{5}{*}[-0.1in]{\rotatebox[origin=c]{90}{gender}}}&  \parbox[t]{2mm}{\multirow{3}{*}[-0.03in]{\rotatebox[origin=c]{90}{\tiny{Fairness}}}}& Data SR& 0.71 (0.02) & 0.97 (0.01) & \textbf{0.98} (0.01) & \textbf{0.97} (0.02) & - & \textbf{0.99} (0.01)  & 0.71 (0.02) \\ 
& & & Data RR  & 0.26 (0.01) & 0.96 (0.01) & \textbf{0.98} (0.01) & \textbf{0.98} (0.01) & - & 0.26 (0.01)  &  \textbf{0.98} (0) \\ 
& & & Clf SR& 0.73 (0.06) & \textbf{0.89} (0.02) & \textbf{0.88} (0.02) & \textbf{0.85} (0.06) & - & 0.79 (0.01)  & 0.73 (0.03) \\ 
\cmidrule{3-11}
& & \parbox[t]{2mm}{\multirow{2}{*}[-0.05in]{\rotatebox[origin=c]{90}{\tiny{Accuracy}}}}& Covariance matrix diff norm & \textbf{0} (0) & 4.64 (0.26) & 3.20 (0.44)  & 5.18 (0.84)  & - & 4.89 (0.04)  & 0.16 (0.01) \\ 
& & & Clf Acc & \textbf{0.65} (0.01) & 0.63 (0.01) & 0.63 (0.01)  & {0.63} (0.01) &  - & 0.62 (0.02)  & 0.63 (0.01) \\ 
\cmidrule{3-11}
& & & Runtime & - & 35s & 40s & 40s &  -   & 0.25s  & 2s  \\ 
\cmidrule{2-11}
& \parbox[t]{4mm}{\multirow{5}{*}[-0.1in]{\rotatebox[origin=c]{90}{race}}}&  \parbox[t]{2mm}{\multirow{3}{*}[-0.03in]{\rotatebox[origin=c]{90}{\tiny{Fairness}}}}& Data SR& 0.73 (0.03) & \textbf{0.98} (0.02) & \textbf{0.98} (0.02) & \textbf{0.97} (0.02) & - & \textbf{0.99} (0)  & 0.72 (0.03) \\ 
& & & Data RR & 0.06 (0) & \textbf{0.99} (0.01) & \textbf{0.99} (0.01) & \textbf{0.99} (0.01) & - & 0.01 (0.01)  &  \textbf{0.98} (0) \\ 
& & & Clf SR& 0.72 (0.01) & \textbf{0.89} (0.06) & \textbf{0.91} (0.06) & \textbf{0.91} (0.05) & - & \textbf{0.85} (0.11)  & 0.71 (0.13) \\ 
\cmidrule{3-11}
& & \parbox[t]{2mm}{\multirow{2}{*}[-0.05in]{\rotatebox[origin=c]{90}{\tiny{Accuracy}}}}& Covariance matrix diff norm & \textbf{0.01} (0) & 1.94 (0.25) & 1.93 (0.24)  & 1.87 (0.26)  & - & 0.88 (0.14)  & 0.36 (0.01) \\ 
&& & Clf Acc & \textbf{0.66} (0.01) & 0.64 (0.01) & 0.64 (0.01)  & {0.63} (0.01) &  - & 0.41 (0.08)  & 0.64 (0.01) \\ 
\cmidrule{3-11}
& & & Runtime & - & 35s & 40s & 40s &  -   & 0.25s  & 2s  \\ 
\bottomrule
\end{tabular}
\label{tab:Clf_DI_Utlity_Compare}
\end{table*}

\section{Conclusion, limitations, and future work}
We present a novel optimization framework that can be used as a data preprocessing method towards mitigating bias.
It works by applying the maximum entropy framework to modified inputs (i.e., the expected vector and prior distribution) which are carefully designed to improve certain fairness metrics.
Using this approach we can learn distributions over large domains, controllably adjust the representation rate or statistical rate of protected groups, yet remains close to the empirical distribution induced by the given dataset.
Further, we show that we can compute the modified distribution in time polynomial in the {\em dimension} of the data.
Empirically, we observe that samples from the learned distribution have desired representation rates and statistical rates, and when used for training a classifier incurs only a slight loss in accuracy while significantly improving its fairness.

Importantly, our pre-processing approach is also useful in settings where group information is not present at runtime or is legally prohibited from being used in classification \cite{edwards16slave}, and hence we only have access to protected group status it in the training set. 
Further, our method has an added privacy advantage of obscuring information about individuals in the original dataset, since the result of our algorithm is a distribution over the domain rather than a reweighting of the actual dataset.

An important extension would be to modify our approach to improve fairness metrics across intersectional types. 
Given multiple protected attributes, one could pool them together to form a larger categorical protected attribute that captures intersectional groups, allowing 
our approach to be used directly. 
However, improving fairness metrics across multiple protected attributes \emph{independently} seems to require additional ideas.
Achieving ``fairness'' in general is an imprecise and context-specific goal. 
The choice of fairness metric depends on the application, data, and impact on the stakeholders of the decisions made, and is beyond the scope of this work. %
However, our approach is not specific to statistical rate or representation rate and can be extended to other fairness metrics by appropriately selecting the prior distribution and expectation vector for our max-entropy framework. 

\begin{figure*}[!htbp]
\centering
\includegraphics[width=\linewidth]{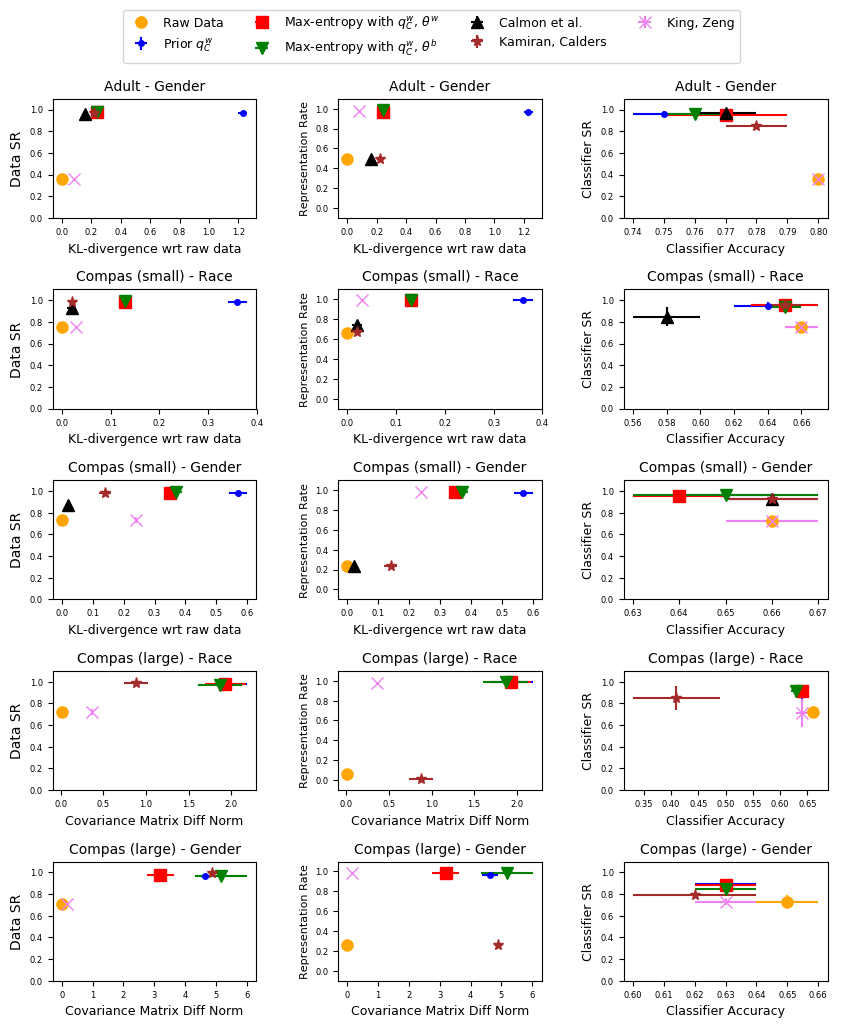}
\caption{\footnotesize The figures represent the fairness (measured using data SR or classifier SR or representation rate) vs accuracy (measured using KL-divergence or covariance matrix difference norm or classifier accuracy) tradeoff for our method and baselines.
``SR'' denotes statistical rate.
For all metrics, we plot the mean across all folds and repetitions, with the standard deviation as error bars.
Note that the approach of \cite{calmon2017optimized} is infeasible for larger domains, such as the large version of COMPAS datasets, and hence we do not present their results on that dataset.
}
\label{fig:main_results}
\end{figure*}

\section{Proofs}

\subsection{Proof of Theorem~\ref{thm:reweighting}} \label{sec:reweighting_appendix}
In this section, we present the proof of the earlier stated properties of the prior distribution and the reweighting algorithm.
Recall that the prior distribution we construct has the following form. For $C \in [0,1]$,
\begin{equation}
 q_C^w(\alpha)= C\cdot u(\alpha) + (1-C) \cdot w(\alpha).
\end{equation}
Here $u$ is the uniform distribution over $\Omega$.
The weight distribution $w$ is obtained using Algorithm \ref{algo:reweighting} to satisfy certain statistical rate constraints. 
To prove Theorem~\ref{thm:reweighting}, we will consider the uniform and weighted part of the $q_C^w$ separately and show that the convex combination of two distributions satisfies similar fairness properties as the two distributions. 
We start with the statements and proofs of bounds for the uniform distribution.
\begin{lemma} \label{lem:prior_fairness_1}
Let $u : \Omega \rightarrow [0,1]$ be the uniform distribution on $\Omega$. 
Then $u$ satisfies the following properties.
\begin{enumerate}
\item For a fixed $y \in \mathcal{Y}$,
$u(Y = y, Z = 0) = u(Y = y, Z = 1).$
\item $u(Z = 0) = u(Z = 1).$
\item For a fixed $y \in \mathcal{Y}$,
$ u(Y = y \mid Z = 0) = u( Y = y \mid Z = 1).$
\end{enumerate}
\end{lemma}
\begin{proof}
(1) For any $\alpha \in \Omega$, let $y(\alpha)$ denote the class label of element $\alpha$ and let $z(\alpha)$ denote the sensitive attribute value of element $\alpha$.
\begin{align*}
u(Y = y, Z=z) =  \sum_{\alpha \in \Omega \; \mid y(\alpha) = y, z(\alpha) = z} \frac{1}{|\Omega|}= \frac{1}{|\Omega|} \cdot \frac{|\Omega|}{2|\mathcal{Y}|} = \frac{1}{2|\mathcal{Y}|}.
\end{align*}
Since the above term is independent of $z$-value, $u(Y = y, Z = z)$ is equal for all $z$.

\noindent
(2) Using 
\[ u(Z = z_1) = \sum_{y \in \mathcal{Y}} u(Z = z_1, Y = y).\]
and part (1), we get
\[  \sum_{y \in \mathcal{Y}} u(Z = z_1, Y = y) = \sum_{y \in \mathcal{Y}} u(Z = z_2, Y = y).\]
This implies that
\[ u(Z = z_1) = u(Z = z_2).\]
(3) Taking the ratio of part (1) and (2), we get
\begin{align*}
u(Y = y \mid Z = z_1) = \frac{u(Y = y, Z = z_1)}{u(Z = z_1)}=  \frac{u(Y = y, Z = z_2)}{u(Z = z_2)}= u(Y = y \mid Z = z_2).
\end{align*}
\end{proof}
\noindent
As expected, the uniform distribution is perfectly fair. We next try to prove similar bounds for the weighted distribution $w$.
\begin{lemma} \label{lem:prior_fairness_2}
Given dataset $\mathcal{S}$ and parameter $\tau \in [0,1]$, let $w$ be the weighted distribution on samples in $\mathcal{S}$  obtained from Algorithm~\ref{algo:reweighting} with input $\mathcal{S}$ and $\tau$.
Then $w$ satisfies the following properties.
\begin{enumerate}
\item For a fixed $y \in \mathcal{Y}$,
$w(Y = y, Z = 0) = \tau \cdot w(Y = y, Z = 1). $
\item $w(Z = 0) = \tau \cdot   w(Z = 1).$
\item For a fixed $y \in \mathcal{Y}$,
$w(Y = y \mid Z = 0) = w( Y = y \mid Z = 1).$
\end{enumerate}
\end{lemma}
\begin{proof}
Note that, by definition, the support of $w$ is the elements in the dataset $\mathcal{S}$. For any $\alpha \in \Omega$, let $y(\alpha)$ denote the class label of element $\alpha$ and let $z(\alpha)$ denote the sensitive attribute value of element $\alpha$.

(1) For any value $z \in \set{0,1}$,
\begin{align*}
w(Z = z, Y = y) =  \sum_{\alpha \in S \; \mid y(\alpha) = y, z(\alpha) = z} w(\alpha) 
\end{align*}
We will analyze the elements with sensitive attribute value 0 and 1 separately since they have different weights. From Algorithm~\ref{algo:reweighting},
\begin{align*}
w( Y = y,  Z = 1) = \sum_{\alpha \in S \; \mid y(\alpha) = y, z(\alpha) = 1} w(\alpha) &=  \sum_{\alpha \in S \; \mid y(\alpha) = y, z(\alpha) = 1}\frac{1}{W} \sum_{i=1}^N \mathbbm{1}(\alpha_i = \alpha) \cdot \frac{c(y)}{c(y,1)}\\
 &=  \frac{1}{W} \cdot c(y,1)  \cdot \frac{c(y)}{c(y,1)} = \frac{ c(y)}{W}.
\end{align*}
Similarly, for elements with sensitive attribute value 0, 
\begin{align*}
w(  Y = y,  Z = 0) =  \sum_{\alpha \in S \; \mid y(\alpha) = y, z(\alpha) = 0} w(\alpha) &=  \sum_{\alpha \in S \; \mid y(\alpha) = y, z(\alpha) = 0}\frac{1}{W} \sum_{i=1}^N \mathbbm{1}(\alpha_i = \alpha) \cdot \frac{\tau \cdot c(y)}{c(y,0)}\\
 &=  \frac{1}{W} \cdot c(y,0)  \cdot \frac{c(y)}{c(y,0)} = \frac{\tau \cdot c(y)}{W}.
\end{align*}
Therefore,
\[\frac{w(Y = y, Z = 0)}{w(Y = y, Z = 1)} = \tau \text{ and } \frac{w(Y = y, Z = 1)}{w(Y = y, Z = 0)} = \frac{1}{\tau} \geq 1. \]
Hence, the ratio for $z_1, z_2$ is atleast $\tau$.

(2) The statement of part (1) holds for all $y \in \mathcal{Y}$. Therefore, 
\[  \sum_{y \in \mathcal{Y}} w(Z = z_1, Y = y) \geq  \tau \cdot \sum_{y \in \mathcal{Y}} w(Z = z_2, Y = y).\]
This implies that
\[ w(Z = z_1) \geq \tau \cdot   w(Z = z_2).\]
Since the probability mass assigned to all sensitive attribute values are within a $\tau$-factor of each other, the representation rate of $w$ is atleast $\tau$.
In particular, using the exact inequalities in the proof of part (1), we get 
\[ \sum_{y \in \mathcal{Y}} w(Z = 0, Y = y) = \tau \cdot \sum_{y \in \mathcal{Y}} w(Z = 1, Y = y)\]
which implies that
\[ w(Z = 0) = \tau \cdot   w(Z = 1).\]
(3) Taking the ratio of part (1) and (2), we get
\[w(Y = y \mid Z = 0) = \frac{w(Y = y, Z = 0)}{w(Z = 0)} =   w(Y = y \mid Z = 1).\]
\end{proof}

\noindent
Before using the above properties of uniform and weighted distribution to prove Theorem~\ref{thm:reweighting}, we will show that the convex combination of two distributions has similar fairness guarantees as the two distributions.

\begin{lemma}[\bf Statistical rate of convex combination of two distributions] \label{lem:prior_stat_combined}
Given distributions $v_1, v_2$ on domain $\Omega$ and a parameter $C \in [0,1]$,  define distribution $q$ as
\[q(\alpha) := C \cdot v_1(\alpha) + (1-C) \cdot v_2(\alpha).\]
For parameters for $0 < \tau_2 \leq \tau_1 \leq 1$, suppose that $v_1, v_2$ satisfy the following properties:
\begin{enumerate}
\item $v_1( Z = 0) = \tau_1 \cdot v_1(Z = 1) \text{ and } ,$
$v_2( Z = 0) = \tau_2 \cdot v_2(Z = 1).$
\item For a fixed $y \in\mathcal{Y}$,
\[v_1(Y=y, Z = 0) = \tau_1 \cdot v_1(Y=y, Z = 1) \text{ and },\]
\[v_2( Y=y, Z = 0) = \tau_2 \cdot v_2(Y=y, Z = 1).\]
\end{enumerate}
Then for a fixed $y \in \mathcal{Y}$ and $z_1, z_2 \in \set{0,1}$, $q$ satisfies the following properties
\begin{enumerate}
\item $q(Y = y \mid Z = z_1) \geq  \tau_1\tau_2\cdot q(Y = y \mid Z = z_2).$
\item \[\frac{q(Y = y, Z = 0)}{q(Y = y, Z = 1)} \geq \tau_2 \text{ and } \frac{q(Y = y, Z = 1)}{q(Y = y, Z = 0)} \geq 1.\]
\end{enumerate}
\end{lemma}

\begin{proof}
From the definition of $q$,
\[q( Z = 0)  = C \cdot v_1(Z = 0) + (1-C) \cdot  v_2(Z = 0).\]
Using the first property of $v_1$ and $v_2$, we get
\begin{align*}
q( Z = 0) &=  C \cdot \tau_1 \cdot v_1(Z = 1) + (1-C) \cdot \tau_2 \cdot  v_2(Z = 1) \\
& = \tau_2 \cdot (C \cdot v_1(Z = 1) + (1-C) \cdot  v_2(Z = 1)) \\ &+ C \cdot (\tau_1-\tau_2) \cdot  v_1(Z = 1) \\
& = \tau_2 \cdot q(Z = 1) + C \cdot (\tau_1-\tau_2) \cdot  v_1(Z = 1) \\
&\geq \tau_2 \cdot q(Z = 1).
\end{align*}
The last inequality holds because $\tau_2 \leq \tau_1$.
Similarly, since $\tau \in (0, 1]$,
\begin{align*}
q( Z = 1) &=  C \frac{1}{\tau_1}\cdot v_1(Z = 0) + (1-C) \cdot \frac{1}{\tau_2} \cdot  v_2(Z = 0)  \\
& \geq \frac{1}{\tau_1} \cdot q(Z = 0) + (1-C) \cdot (\frac{1}{\tau_2} - \frac{1}{\tau_1}) \cdot  v_1(Z = 0) \\
&\geq \frac{1}{\tau_1} \cdot q(Z = 0).
\end{align*}
In other words, the representation rate of $q$ is atleast $\tau_2$.
Once again, using the definition of $q$,
\begin{align*}
q(Y = y, Z = 0) &= C \cdot v_1(Y = y, Z = 0)\\ &+ (1-C)\cdot  v_2(Y = y, Z = 0). 
\end{align*}
Using the properties of $v_1, v_2$, we can alternately write the above expression as
\begin{align*}
q(Y = y, Z = 0) &= C \cdot \tau_1 \cdot v_1(Y = y, Z = 1)\\ &+ (1-C)\cdot  \tau_2 \cdot v_2(Y = y, Z = 1). 
\end{align*}
Let $a = C \cdot v_1(Y = y, Z = 1) \text{ and } b = (1-C)\cdot v_2(Y = y, Z = 1). $
Then,
\begin{align*}
\frac{q(Y = y, Z = 0)}{q(Y = y, Z = 1)}  = \frac{a\tau_1 + b\tau_2}{a + b} = \tau_2 + \frac{(\tau_1 - \tau_2) a}{a+b} \geq \tau_2,
\end{align*}
since $a, b, (\tau_1 - \tau_2) \geq 0$. Similarly, since $\tau_1, \tau_2 \in [0,1]$
\begin{align*}
\frac{q(Y = y, Z = 1)}{q(Y = y, Z = 0)}  = \frac{a + b}{a\tau_1 + b\tau_2} \geq 1.
\end{align*}
Hence the ratio of the joint distributions for different values of sensitive attributes is atleast $\tau$.
Now to prove the statistical rate bound, we just need to take the ratio of the joint distribution and marginal distribution.
Taking the ratio we get,
\begin{align*}
q(Y= y \mid  Z = 0) = \frac{q(Y= y,  Z = 0)}{q(Z = 0)} \geq \frac{\tau_2 \cdot q(Y= y,  Z = 1)}{\frac{1}{\tau_1} q(Z = 1)} = \tau_1\tau_2 \cdot q(Y= y \mid  Z = 1)  .
\end{align*} 
Similarly,
\begin{align*}
q(Y= y \mid  Z = 1) = \frac{q(Y= y,  Z = 1)}{q(Z = 1)} \geq \frac{q(Y= y,  Z = 0)}{\frac{1}{\tau_2} \cdot q(Z = 0)}= \tau_2 \cdot q(Y= y \mid  Z = 0)  .
\end{align*} 
Since $\tau_2 \leq \tau_1 \leq 1$, the minimum of the two ratios is $\tau_1\tau_2$. Hence the statistical rate of $q$ is  $\tau_1\tau_2$.

\end{proof}

\noindent
While the first result of the above lemma bounds the statistical rate of $q$, the second result will be useful in bounding the statistical rate of the max-entropy distribution obtained using $q$.
Using Lemma~\ref{lem:prior_stat_combined}, we can now prove the representation rate and statistical rate bound on the prior $q_C^w$.
\begin{proof}[Proof of Theorem~\ref{thm:reweighting}]
Proving the first statement is simple.
Since Algorithm~\ref{algo:reweighting} just counts the number of elements in $\mathcal{S}$ satisfying certain properties, the time taken is $|\mathcal{Y}| \cdot N$. In case of hypercube domain, $|\mathcal{Y}| = 2$. Hence the time complexity of the re-weighting algorithm is linear in $N$.

For the statistical rate of $q_C^w$, plugging $v_1 =u$ and $v_2 = v^w$ in Lemma~\ref{lem:prior_stat_combined}, we can get the corresponding ratio for $q_C^w$. In particular, from Lemma~\ref{lem:prior_fairness_1} and Lemma~\ref{lem:prior_fairness_2}, we know that $\tau_1 = 1$ for distribution $u$ and $\tau_2 = \tau$ for distribution $v^w$.
The statement of Lemma~\ref{lem:prior_stat_combined} then tells us that the statistical rate of $q_C^w$ is atleast $\tau$.
\end{proof}

\subsection{Proof of Lemma~\ref{lem:bounding_box}} \label{sec:proof_bounding_box}

In this section, we provide the proof of the bound on the size of the optimal dual solution.
\noindent

\begin{proof}[Proof of Lemma~\ref{lem:bounding_box}]
The proof of this lemma is along similar lines as the proof of bounding box in \cite{SinghV14}.
The key difference is that the proof in \cite{SinghV14} does not consider a prior on the distribution.
We are given that $\theta$ is in the $\eta$-interior of the hypercube, i.e., for each $1\leq i \leq d$, 
$\eta < \theta_i < 1-\eta$.
Hence a ball of radius $\eta$, centered at $\theta$, is contained with the hypercube.

We will first provide a bound for a general prior $q$ and then substitute properties specific to $q_C^w$.
To that end, for a prior $q$ let $L_q$ denote the following quantity,
\[L_q := \log \frac{1}{\min_{\alpha} q(\alpha)}.\]
To show the bound in Lemma~\ref{lem:bounding_box}, we will try to prove that the optimal dual solution, multiplied by a factor of $\nicefrac{1}{L_q}$, lies in a ball of radius $\nicefrac{1}{\eta}$ centered at $\theta$ and later provide a bound on $L_q$.
Let
\[\hat{\lambda} = \theta - \frac{\lambda^\star}{L_q}.\]
Firstly, note that we can bound the objective function of $\eqref{eq:dual-program}$ as follows.
Since the objective function of \eqref{eq:primal-program}  is the negative of KL-divergence, it's value is always less than zero.
Hence, by strong duality we get that, for a given prior $q$,
\[\log \left(\sum_{\alpha \in \set{0,1}^d} q(\alpha) e^{\inner{\alpha - \theta, \lambda^\star}} \right) \leq 0. \]
This implies that
\begin{align*} 
\min_{\alpha} q(\alpha) \sum_{\alpha \in \set{0,1}^d} e^{\inner{\alpha - \theta, \lambda^\star}} \leq  \sum_{\alpha \in \set{0,1}^d} q(\alpha) e^{\inner{\alpha - \theta, \lambda^\star}} \leq 1.
\end{align*}
Therefore, for all $\alpha \in \set{0,1}^d$,
\[e^{\inner{\alpha - \theta, \lambda^\star}} \leq \frac{1}{\min_{\alpha} q(\alpha)}.\]
Taking log both sides, we get
\[\inner{\alpha - \theta, \lambda^\star} \leq \log \frac{1}{\min_{\alpha} q(\alpha)} = L_q.\]
Substituting $\hat{\lambda}$, we get 
\[\inner{\alpha - \theta, \theta - \hat{\lambda}} \leq 1. \eqlabel{1}\]
Note that since this inequality holds for all $\alpha \in \set{0,1}^d$, it also holds for all $\alpha \in \conv{\set{0,1}^d}$.
Next we choose $\alpha$ appropriately so as to bound the distance between $\theta$ and $\hat{\lambda}$.
Choose 
\[\alpha = \theta + \frac{\theta - \hat{\lambda}}{\norm{\theta - \hat{\lambda}}}\cdot \eta.\]
Note that $\norm{\alpha - \theta} \leq \eta$, hence this $\alpha$ lies within the hypercube. Then we can apply \eqref{1} to get
\[\left\langle {\frac{\theta - \hat{\lambda}}{\norm{\theta - \hat{\lambda}}}\cdot \eta, \theta - \hat{\lambda}} \right\rangle \leq 1.\]
This directly leads to
\[\norm{\theta - \hat{\lambda}} \leq \frac{1}{\eta}.\]
Hence we know that $\hat{\lambda}$ is within a ball of radius $\nicefrac{1}{\eta}$ centered at $\theta$.
Substituting the definition of $\hat{\lambda}$ into this bound, we directly get that
\[\left\lVert{\frac{\lambda^\star}{L_q}}\right\rVert \leq \frac{1}{\eta} \implies \norm{\lambda^\star} \leq \frac{L_q}{\eta}. \eqlabel{2}\]
The above bound is generic for any given prior $q$. To substitute $q = q_C^w$, we simply need to calculate $L_{q_C^w}$.
Note that the prior $q_C^w$ assigns a uniform probability mass to all points not in the dataset $\mathcal{S}$.
Hence, for any $\alpha \in \set{0,1}^d$
\[q_C^w(\alpha) \geq \frac{C}{|\Omega|} = \frac{C}{2^d}.\]
Therefore, 
\[L_{q_C^w} \leq d \log\frac{1}{C}.\]
Substituting the value of $L_{q_C^w}$ in \eqref{2}, we get
\[\norm{\lambda^\star} \leq \frac{d}{\eta}\log\frac{1}{C}.\]

\end{proof}

\noindent
We note that, for our applications, we can show that the assumption on $\theta$ in the lemma follows from an assumption on the ``non-redundancy'' of the data set. 
Using recent results of \cite{straszakMED2019}, we can get around this assumption and we omit the details from this version of the paper.

\paragraph{Interiority of expected vector. }
The assumption that $\theta$ should be in $\eta$-interior the hypercube can translate to an assumption on the ``non-redundancy'' of the data set, for some natural choices of $\theta$.
For example, to maintain consistency with the dataset $\mathcal{S}$, $\theta$ can be set to be the following:
\[\theta = \sum_{\alpha \in \mathcal{S}} \frac{n_\alpha}{N} \alpha.\]
This corresponds to the mean of the dataset.
In this case, the assumption that for each $1\leq i \leq d$, 
\[\eta < \theta_i\]
implies that more than $\eta$-fraction of the elements in the dataset $\mathcal{S}$ have the $i$-th attribute value 1.
Similarly,
\[\theta_i > 1 - \eta\]
implies that more than $\eta$-fraction of the elements in the dataset $\mathcal{S}$ have the $i$-th attribute value 0.
The reason that this is a non-redundancy assumption is that it implies that no attribute is redundant in the dataset.
For example, if for an attribute $i$, $\theta_i$ was 1 it would mean that all elements in $\mathcal{S}$ have the $i$-th attribute 1 and in that case, we can simply remove the attribute.

\subsection{Proof of Lemma~\ref{lem:counting-oracles}} \label{sec:proof_counting_oracle}

Next the proof of efficient dual oracles is provided here.

\begin{proof}[Proof of Lemma~\ref{lem:counting-oracles}]
For the given prior $q$ and vector $\theta$, let $g_{\theta,q}$ denote the sum, i.e.,
\[g_{q}(\lambda) :=\sum_{\alpha\in\Omega} q(\alpha)e^{\inner{\alpha,\lambda}} \]
Then the dual function $h_{\theta,q}(\lambda)$ is
\[h_{\theta,q}(\lambda) =\log \left( g_{q}(\lambda) \right) - \inner{\theta, \lambda}.\]
The main bottleneck in computing the above quantities is evaluating the summation terms. For all three terms, the summation is obtained from the derivative of $g_q$.
\[\nabla g_{q}(\lambda) = \sum_{\alpha\in\Omega}  \alpha  \cdot q(\alpha)e^{\inner{\alpha,\lambda}} \text{ and } \]
\[\nabla^2 g_{q}(\lambda) = \sum_{\alpha\in\Omega}  \alpha\alpha^\top  \cdot q(\alpha)e^{\inner{\alpha,\lambda}}.\]
Then, the gradient and Hessian can be represented using $\nabla g_{q}$ and $\nabla^2 g_{q}$.
\begin{align*}
\nabla h_{\theta,q}(\lambda)  &= \frac{1}{g_{q}(\lambda)} \nabla g_{q}(\lambda) - \theta,
\end{align*}
\begin{align*}
\nabla^2 h_{\theta,q}(\lambda)  = \frac{1}{g_{q}(\lambda)} \nabla^2 g_{q}(\lambda) - \frac{1}{g_{q}(\lambda)^2} \nabla g_{q}(\lambda)\nabla g_{q}(\lambda)^\top.
\end{align*}

\noindent
Given the above representation of gradient and oracle, if we are able to compute $g_{q}(\lambda), \nabla g_{q}(\lambda), \nabla^2 g_{q}(\lambda)$ efficiently, then using these to compute $h_{\theta,q}(\lambda)$, $\nabla h_{\theta,q}(\lambda)$ and $\nabla^2 h_{\theta,q}(\lambda)$ just involves constant number of addition and multiplication operations, time taken for which is linear in bit complexities of the numbers involved.
Hence we will focus on efficiently evaluating the summations.
Recall that 
\[q = q_C^w = C \cdot u + (1-C) \cdot w.\]
Since $g_{q}(\lambda), \nabla g_{q}(\lambda), \nabla^2 g_{q}(\lambda)$ are all linear in $q$, we can evaluate the summations separately for $u$ and $w$.

For $w$, since the support of the distribution is just the dataset $\mathcal{S}$,
\begin{align*}
g_{w}(\lambda) = \sum_{\alpha\in\Omega} w(\alpha)e^{\inner{\alpha,\lambda}} 
 = \sum_{\alpha\in S} w(\alpha)e^{\inner{\alpha,\lambda}}
\end{align*}
We can directly evaluate the summation using $O(Nd)$ operations (first compute the inner product then summation), where each operation is linear in the bit complexity of $w$ and $e^\lambda$.
For $\nabla g_{w}(\lambda)$, we can represent it as
\[g_{w}(\lambda)  = \sum_{\alpha\in S} \alpha \cdot w(\alpha)e^{\inner{\alpha,\lambda}}.\]
Once again we can evaluate all inner products using $O(Nd)$ operations and then compute the gradient vector in another $O(Nd)$ operations.
In a similar manner, we can also evaluate $\nabla^2 g_{w}(\lambda)$ in $O(Nd^2)$ operations.

Next we need bounds on the number of operations required for the uniform part of $q$.
The main idea is that if the distribution is uniform over the entire domain, then the summation can be separated in terms of the individual features.
For the uniform distribution, let us write $\lambda$ as $(\lambda_1,\ldots,\lambda_d)$, where $\lambda_i$ corresponds to $i$th attribute and let us define variables:
\[\overline{\alpha}_i := \alpha_i \cdot e_i,\]
where $e_i$ is the standard basis vector in $\R^d$, with 1 in the $i$-th location and 0 elsewhere. Let
\begin{align*}
s_i^0 :=& \sum_{\alpha_i\in \set{0,1}} e^{\lambda_i \cdot\alpha_i},\nonumber\\
s_i^1 :=& \sum_{\alpha_i\in \set{0,1}} \overline{\alpha}_i e^{\lambda_i \cdot \alpha_i},\nonumber\\
s_i^2 :=& \sum_{\alpha_i\in \set{0,1}} \overline{\alpha}_i\overline{\alpha}_i^\top e^{\lambda_i \cdot \alpha_i},\nonumber
\end{align*}
for all $i\in \set{1, \dots, d}$ and $\alpha_i\in \set{0,1}$.
Next, we can compute the $g_{u} (\lambda), \nabla g_{u}(\lambda), \nabla^2 g_{u}(\lambda)$ using these values.
\begin{align*}
g_{u} (\lambda)= \frac{1}{|\Omega|}\sum_{\alpha\in\Omega} e^{\inner{\alpha,\lambda}}
	=& \frac{1}{\abs{\Omega}} \prod_{i=1}^d s_i^0,\nonumber\\
\nabla g_{u} (\lambda) = \frac{1}{|\Omega|}\sum_{\alpha\in\Omega} \alpha \cdot e^{\inner{\alpha,\lambda}}
	=& \frac{1}{\abs{\Omega}} \sum_{i=1}^d \left(s_i^1 \prod_{j\neq i} s_j^0\right),\nonumber
\end{align*}
\begin{align*}
\nabla^2 g_{u} (\lambda) = \frac{1}{|\Omega|}\sum_{\alpha\in\Omega} \alpha\alpha^\top \cdot e^{\inner{\alpha,\lambda}} \nonumber
	= \frac{1}{\abs{\Omega}} \sum_{i=1}^d \left[s_i^2 \prod_{j\neq i} s_j^0 + \sum_{j\neq i} s_i^1 (s_j^1)^\top \prod_{k\neq i,j}s_k^0\right].\nonumber
\end{align*}
Evaluating $g_{u} (\lambda)$  involves $(d-1)$ multiplication operations.
Similarly, evaluating $\nabla g_{u} (\lambda)$ involves $O(d^2)$ addition and multiplication operations.
Finally, evaluating $\nabla^2 g_{u} (\lambda)$ involves $O(d^3)$ addition and multiplications operations. 
Each operation takes time polynomial in the bit complexity of $e^\lambda$.

We have shown that for both parts $u$ and $w$, evaluating the above summations takes time polynomial in the bit complexities of the numbers involved.
Since $q$ is a convex combination of $u$ and $w$, computing $g_{u} (\lambda)$, $\nabla g_{u} (\lambda)$ and $\nabla^2 g_{u} (\lambda)$ also takes time polynomial in the bit complexities of the numbers involved.
Specifically, computing $g_{u} (\lambda)$ requires $O(Nd)$  operations, computing $\nabla g_{u} (\lambda)$ requires $O(d(N+d))$  operations and computing $g_{u} (\lambda)$ requires $O(d^2(N+d))$  operations.

\end{proof}

\subsection{Proof of Theorem~\ref{thm:statistical_rate_bound}} \label{sec:proof_stat_rate}
Finally, the proof of the statistical rate guarantee is given in this section.

\begin{proof}[Proof of Theorem~\ref{thm:statistical_rate_bound}]
The proof of this theorem uses the bounds on the distribution of $q_C^w$ that are obtained from Lemma~\ref{lem:prior_stat_combined}.
By the definition of $\delta$, we have that
\begin{align*}
q_C^w(Y=y, Z=z) - \delta \leq p^\star(Y=y, Z=z)  \leq q_C^w(Y=y, Z=z) + \delta.
\end{align*}
\noindent
Using this inequality, we can bound the ratio of the above term for different sensitive attributes as
\[\frac{p^\star(Z = z_1, Y = y)}{p^\star(Z = z_2, Y = y)} \geq  \frac{q_C^w(Y=y, Z=z_1) - \delta}{q_C^w(Y=y, Z=z_2) + \delta}.\]
Next, applying Lemma~\ref{lem:prior_stat_combined}, with $v_1 = u$ and $v_2 = v^w$, we have the following properties of $q_C^w$
\[ q_C^w(Y = y, Z = 0)  \geq \tau \cdot q_C^w(Y = y, Z = 1),\]
and 
\[ q_C^w(Y = y, Z = 1)  \geq q_C^w(Y = y, Z = 0).\]
Furthermore, since $q_C^w$ assigns a uniform mass to all points in $\Omega$, we can also get a lower bound on $q_C^w(Y=y, Z=z_2) $.
\begin{align*}
 q_C^w(Y=y, Z=z) = \sum_{\alpha \mid y(\alpha) = y, z(\alpha) = z} q_C^w(\alpha) 
 \geq \sum_{\alpha \mid y(\alpha) = y, z(\alpha) = z} \frac{C}{|\Omega|} = \frac{C}{2 |\mathcal{Y}|}.
\end{align*}
We can now use the fairness guarantee on $q_C^w$ and lower bound for distribution to get the ratio bounds for max-entropy distribution.
\begin{align*}
\frac{p^\star(Y = y, Z = 0)}{p^\star(Y = y, Z = 1)} &\geq  \frac{\tau \cdot q_C^w(Y=y, Z=1) - \delta}{q_C^w(Y=y, Z=1) + \delta} \\
& = \tau - \delta \cdot \frac{(1 + \tau)}{q_C^w(Y=y, Z=1) + \delta} \\
&\geq \tau - \delta \cdot \frac{(1 + \tau)}{\frac{C}{2|\mathcal{Y}|} + \delta}.
\end{align*}
By the choice of $\theta$, we know that 
\[1 - \theta_\ell > \theta_\ell \implies p^\star(Z = 1) \geq p^\star(Z=0).\]
Therefore,
\begin{align*}
\frac{p^\star(Y = y \mid Z = 0)}{p^\star(Y = y \mid  Z = 1)} = \frac{p^\star(Y = y, Z = 0)}{p^\star(Y = y, Z = 1)} \cdot \frac{p^\star(Z = 1)}{p^\star(Z = 0)} \geq \tau - \delta \cdot \frac{(1 + \tau)}{\frac{C}{2|\mathcal{Y}|} + \delta}.    
\end{align*}
Similarly, for the other direction of this ratio, we can get
\begin{align*}
\frac{p^\star(Y = y, Z = 1)}{p^\star(Y = y, Z = 0)} &\geq  \frac{q_C^w(Y=y, Z=0) - \delta}{q_C^w(Y=y, Z=0) + \delta} \\
& = 1 - \delta \cdot \frac{2}{q_C^w(Y=y, Z=0) + \delta}\\
&\geq  1 - \delta \cdot \frac{2}{\frac{C}{2|\mathcal{Y}|} + \delta}.
\end{align*}
Once again, 
\[1 - \theta_\ell > \tau \cdot \theta_\ell \implies p^\star(Z = 0) \geq p^\star(Z=1).\]
Therefore,
\begin{align*}
\frac{p^\star(Y = y \mid Z = 1)}{p^\star(Y = y \mid  Z = 0)} = \frac{p^\star(Y = y, Z = 1)}{p^\star(Y = y, Z = 0)} \cdot \frac{p^\star(Z = 0)}{p^\star(Z = 1)} \geq \tau \left( 1 - \delta \cdot \frac{2}{\frac{C}{2|\mathcal{Y}|} + \delta} \right).
\end{align*}
Note that
\[\tau \left( 1 - \delta \cdot \frac{2}{\frac{C}{2|\mathcal{Y}|} + \delta} \right) \geq \tau - \delta \cdot \frac{(1 + \tau)}{\frac{C}{2|\mathcal{Y}|} + \delta}.\]
Using $|\mathcal{Y}| = 2$, we get that the statistical rate is atleast
\[\tau -\frac{  4\delta \cdot (1 + \tau)}{C + 4\delta}.\]

\end{proof}

\section*{Acknowledgements}
This research was supported in part by NSF CCF-1908347 and an AWS MLRA Award. 
We thank Ozan Yildiz for initial discussions on algorithms for max-entropy optimization.

\bibliographystyle{plain}
\bibliography{refs} 
\clearpage
\newpage

\appendix

\section{Sampling oracle} \label{sec:sampling_oracle}
As stated earlier, the max-entropy distribution $p^\star$ can be succinctly represented using the solution of the dual program $\lambda^\star$.
In particular, we have that
$$p^\star(\alpha) = \frac{q(\alpha) e^{\inner{\lambda^\star,\alpha}}}{\sum_{\beta\in\Omega} q(\beta) e^{\inner{\lambda^\star,\beta}}}.$$ 
Using the efficient counting oracles of Lemma~\ref{lem:counting-oracles} and bounding box of Lemma~\ref{lem:bounding_box}, we efficiently compute a good approximation to the dual solution $\lambda^\star$.
But sampling from the distribution $p^\star$ can still be difficult due to the large domain size.
In this section, we show that given $\lambda^\star$ we can efficiently sample from the max-entropy distribution $p^\star$ using the counting oracles described earlier.

\begin{theorem}[\bf Sampling from counting] \label{thm:sampling}
There is an algorithm that, given a weighted distribution $w : \mathcal{S} \to [0,1]$ and $\lambda \in \R^d$, returns a sample from the distribution $p$, where for any $\alpha \in \Omega$
\[p(\alpha) = \frac{q_C^w(\alpha) e^{\inner{\lambda,\alpha}}}{\sum_{\beta\in\Omega} q_C^w(\beta) e^{\inner{\lambda,\beta}}}.\]
The running time of this algorithm is polynomial in $N, d$ and bit complexities of all numbers involved: $w(\alpha)$ for $\alpha \in \mathcal{S}$ and $e^\lambda_i$, for $i \in \set{1, \dots, d}$.
\end{theorem}
\noindent
The equivalence of counting and sampling is well-known and a very useful result \cite{jerrum1986random}.
We provide the proof for our setting here, for the sake of completion.
\begin{proof}
As mentioned before, the goal is to sample from the distribution
\[p(\alpha) = \frac{q_C^w(\alpha) e^{\inner{\lambda,\alpha}}}{\sum_{\beta\in\Omega} q_C^w(\beta) e^{\inner{\lambda,\beta}}}.\]
The primary bottleneck in sampling is evaluating the normalizing term,
\[\sum_{\beta\in\Omega} q_C^w(\beta) e^{\inner{\lambda,\beta}}.\]
To evaluate this sum, we have an efficient oracle, i.e, the counting oracle from Lemma~\ref{lem:counting-oracles}.
The lemma (and the algorithm) allow us to calculate the sum in $O(Nd)$ operations, where each operation has bit complexity polynomial in the numbers involved: $w(\alpha)$ for $\alpha \in \Omega$ and $e^{\lambda}$.
Hence, we can evaluate the normalizing term efficiently.

However, we still cannot sample by enumerating all probabilities since the size of the domain is exponential. 
To efficiently sample from the distribution, we sample each feature of $\alpha$ individually. 
Let $A$ denote the random variable with probability distribution $p$. 
Let $A_1$ denote the element at the first position of $A$.

\begin{align*}
    \mathbb{P}[A_1 = 0] = \frac{\sum_{\alpha\in\Omega \mid \alpha_1 = 0} q_C^w(\alpha) e^{\inner{\lambda,\alpha}}}{\sum_{\beta\in\Omega} q_C^w(\beta) e^{\inner{\lambda,\beta}}} 
    = \frac{\sum_{\hat{\alpha} \in\Omega^{(1)}} q_{C,1}^w(\hat{\alpha}) e^{\inner{\lambda^{(1)},\hat{\alpha}}}}{\sum_{\beta\in\Omega} q_C^w(\beta) e^{\inner{\lambda,\beta}}}.
\end{align*}
Here $\lambda^{(1)}$ is $\lambda$ without the first element, $\Omega^{(1)}$ is the subdomain of all feature except the first feature and $q_{C,1}^w$ is the distribution $q_C^w$ conditional on the first feature being always 0.
Note that $q_{C,1}^w$ is a distribution supported on $\Omega^{(1)}$, and we can use the counting oracle of Lemma~\ref{lem:counting-oracles} to calculate the sum
\[\sum_{\hat{\alpha} \in\Omega^{(1)}} q_{C,1}^w(\hat{\alpha}) e^{\inner{\lambda^{(1)},\hat{\alpha}}}\]
in $O(N(d-1))$ operations.
Hence we can calculate the probability $\mathbb{P}[A_1 = 0]$ in $O(Nd)$ operations. Then we can do a coin toss, whose tail probability is chosen to be $\mathbb{P}[A_1 = 0]$, and set $\alpha_1 = 1$ if we heads and $\alpha_1=0$ otherwise.
Next depending on the value we get for $\alpha_1$, we can calculate the marginal probability of $\alpha_2$ being 0.
Say $\alpha_1 = a_1$. Then
\begin{align*}
    \mathbb{P}[A_1 = 0] &= \frac{\sum_{\alpha\in\Omega \mid \alpha_1 = a_1, \alpha_2 = 0} q_C^w(\alpha) e^{\inner{\lambda,\alpha}}}{\sum_{\beta\in\Omega \mid \beta_1 = a_1} q_C^w(\beta) e^{\inner{\lambda,\beta}}}.
\end{align*}
We can repeat the above process of calculating these summations using the counting oracle and once again sample a value of $\alpha_2$ using the biased coin toss.
Repeating this process $d$ times, we get a sample from the distribution $p$.
The number of operations required is $O(Nd^2)$, where each operation has bit complexity polynomial in the numbers involved: $w(\alpha)$ for $\alpha \in \Omega$ and $e^{\lambda}$.

\end{proof}

\begin{figure*}[t]
\centering
  \includegraphics[width=\linewidth]{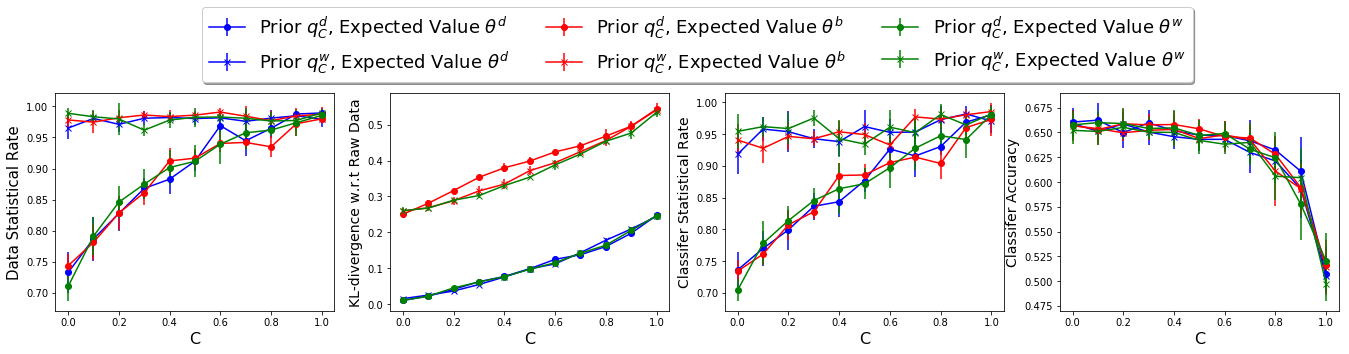}
  \subfloat[Data statistical rate vs C]{\hspace{.25\linewidth}}
\subfloat[KL-divergence w.r.t. raw data vs C ]{\hspace{.25\linewidth}}
  \subfloat[Classifier Statistical rate vs C   ]{\hspace{.25\linewidth}}
\subfloat[Classifier Accuracy vs C ]{\hspace{.25\linewidth}}\\
\caption{ {Comparison of max-entropy distributions with different priors and expectation vectors for small version of COMPAS dataset. Note that a value of C = 1 effectively would result in sampling uniformly at random from the entire domain. Hence, as expected, we see fairness increase and accuracy decrease as C increases.
(a) Data statistical rate for COMPAS dataset. We observe that  using $q_C^{\textrm{w}}$ is better with respect to statistical rate than using  $q_C^{\textrm{d}}$. The value of C does not significantly affect the results for  $q_C^{\textrm{w}}$; this is expected since $q_C^w$ is constructed to be fair for all $C$.
(b) KL-divergence between the empirical distributions as compared with the raw COMPAS data. We observe that this value is smaller when using the expected vector $\theta^{\textrm{d}}$.
(c) Classifier statistical rate vs C. Similar to data statistical rate results for COMPAS dataset, we observe that using the  $q_C^{\textrm{w}}$ prior results in a fairer outcome. Here there is a slight increase in fairness as C is increased even for $q_C^{\textrm{w}}$.
(d) Classifier accuracy vs C. We observe that there is no significant difference in accuracy across different metrics and priors. This is surprising, especially in light of the significant differences with respect to how well they capture the raw data.%
}}
\label{fig:maxEnt_compare}
\end{figure*}

\begin{figure*}[t]
\centering
  \includegraphics[width=\linewidth]{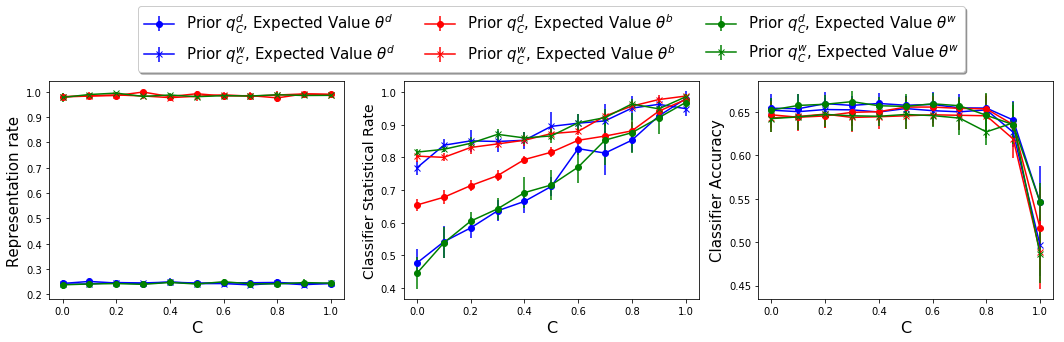}
  \subfloat[Representation rate vs C.]{\hspace{.33\linewidth}}
\subfloat[ Classifier Statistical Rate vs C.]{\hspace{.33\linewidth}}
  \subfloat[Classifier Accuracy vs C.]{\hspace{.33\linewidth}}
\caption{ The figures show the comparison of max-entropy distributions with different prior distributions and expected values. The base dataset is the small version of COMPAS.
The first figure show the representation rate of different max-entropy distribution; the representation rate is 1 when using balanced expected vectors, such as $\theta^w$ or $\theta^b$. The second and third figure show the statistical rate and accuracy of Gaussian Naive Bayes classifier trained on the output distribution.
While the trend across different parameters is the same as observed using decision tree classifier, we note that in this case, the classifier statistical rate is relatively smaller for smaller values of $C$.
}
\label{fig:maxEnt_compare_other}
\end{figure*}

\begin{figure*}[t]
\centering
  \includegraphics[width=\linewidth]{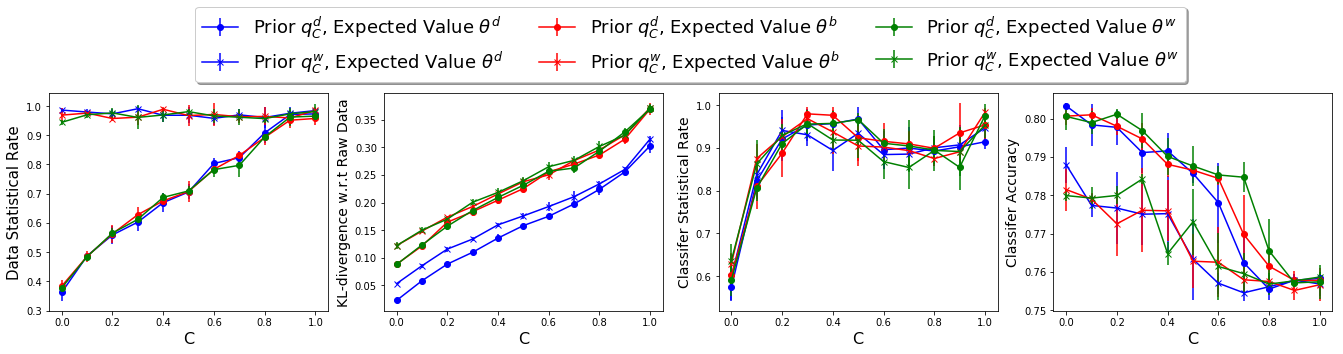}
  \subfloat[Data statistical rate vs C ]{\hspace{.25\linewidth}}
\subfloat[KL-divergence w.r.t. raw data vs C ]{\hspace{.25\linewidth}}
  \subfloat[Classifier Statistical rate vs C   ]{\hspace{.25\linewidth}}
\subfloat[Classifier Accuracy vs C ]{\hspace{.25\linewidth}}\\
\caption{ {Comparison of max-entropy distributions with different priors and expectation vectors for small version of Adult dataset.
(a) Data statistical rate for Adult dataset. Once again using $q_C^{\textrm{w}}$ is better with respect to statistical rate than using  $q_C^{\textrm{d}}$. %
(b) KL-divergence between the empirical distributions as compared with the raw Adult data. We observe that this value is smaller when using the expected vector $\theta^{\textrm{d}}$.
However, in this case the gap between divergence when using $q_C^w$ and divergence when using $q_C^d$ is smaller than observed with COMPAS.
(c) Classifier statistical rate vs C. In this case, using even $q_C^d$ achieves relatively good statistical rate. However, the statistical rate of max-entropy distributions using $q_C^w$ is slightly better in most cases.
(d) Classifier accuracy vs C. As expected, classifier accuracy is higher for distributions using $q_C^d$ than distributions using $q_C^w$. This is because $q_C^w$ involves weighing the samples in a manner that is not always consistent with the frequency of the samples.
}}
\label{fig:maxEnt_compare_clf}
\end{figure*}

\begin{figure*}[t]
\includegraphics[width=\textwidth]{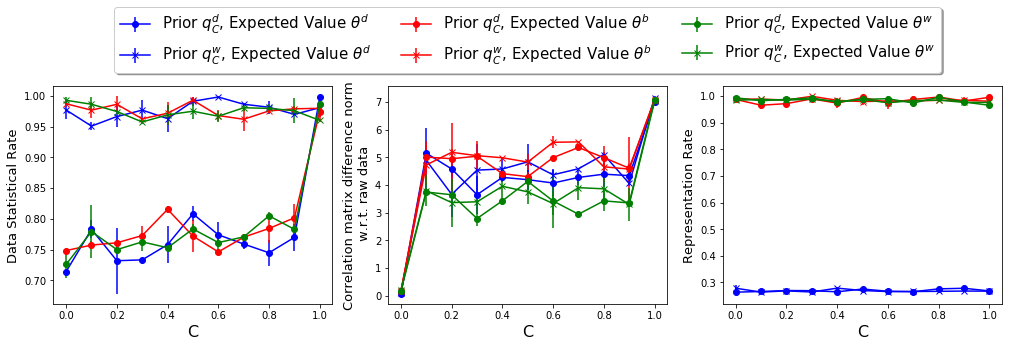}
  \subfloat[Data statistical rate vs C ]{\hspace{.33\linewidth}}
\subfloat[Utility vs C ]{\hspace{.33\linewidth}}
  \subfloat[Representation rate vs C]{\hspace{.33\linewidth}}
\caption{Comparison of statistical rate, representation rate and correlation matrix difference with respect to raw data for max-entropy distributions with different priors and expected values.
The base dataset is the large version of COMPAS.}
\label{fig:DI_Utlity_MaxEnt_large}
\end{figure*}

\begin{figure*}[t]
\includegraphics[width=\textwidth]{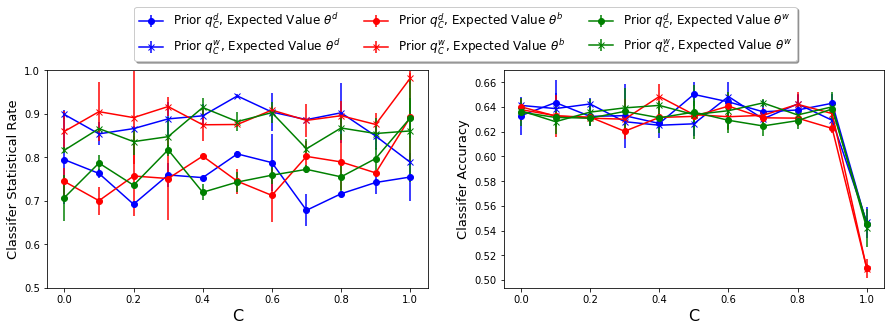}
  \subfloat[Classifier statistical rate vs C ]{\hspace{.5\linewidth}}
\subfloat[Classifier accuracy vs C ]{\hspace{.5\linewidth}}
\caption{Comparison of Decision Tree classifier trained on data from different max-entropy distributions with different prior distributions and expected values. The base dataset is the large version of COMPAS.}
\label{fig:Clf_DI_Utlity_MaxEnt_DTC_large}
\end{figure*}

\begin{figure*}[t]
\includegraphics[width=\textwidth]{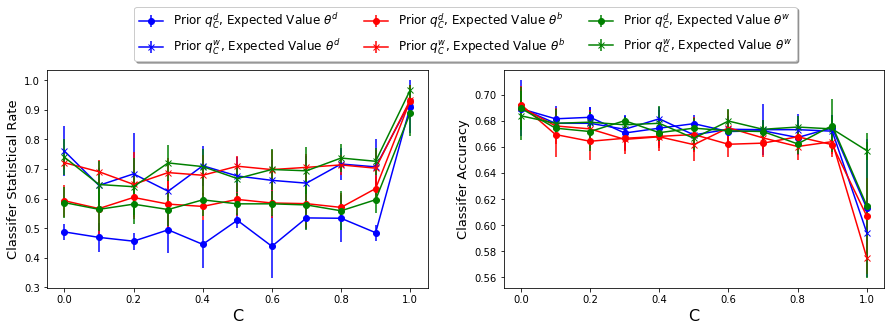}
  \subfloat[Classifier statistical rate vs C ]{\hspace{.5\linewidth}}
\subfloat[Classifier accuracy vs C ]{\hspace{.5\linewidth}}
\caption{Comparison of Gaussian Naive Bayes classifier trained on data from max-entropy distributions with different prior distributions and expected values. The base dataset is the large version of COMPAS.}
\label{fig:Clf_DI_Utlity_MaxEnt_GNB_large}
\end{figure*}

\section{Additional details and empirical results for small COMPAS and Adult datasets}\label{sec:experiments1}

\noindent
\paragraph{Features of Adult dataset.}
The demographic features used from this dataset are gender, race, age and years of education.
The age attribute  in this case is categorized by decade, with 7 categories (the last one being age $\geq 70$ years).
The education years attribute is also a categorical attribute, with the categories being $(< 6), 6, 7, \cdots, 12, (>12)$ years.
The label is a binary marker indicating whether the annual income is greater than \$50K or not.

\noindent
\paragraph{Features of small version of COMPAS dataset.}
For this dataset, we use the features gender, race, age, priors count, and charge degree as features, and a binary marker of recidivism within two years as the label.

Given training data $\mathcal{S}$, we can estimate different maximum entropy distributions with given parameters using $\mathcal{S}$.
We use two kinds of prior distributions: 
(1)  $q_C^{\textrm{d}}$ assigns uniform weights to the samples, i.e., $w = \set{\nicefrac{n_\alpha}{N}}_{\alpha \in \mathcal{S}}$, and 
(2) $q_C^{\textrm{w}}$ assigns weights returned by the  Algorithm~\ref{algo:reweighting} (also used for results in Table~\ref{tab:Clf_DI_Utlity_Compare}).

We use three kinds of expectation vectors: 
(a) the expected value of the dataset $\mathcal{S}$,
$$\theta^d := \left(\sum_{\alpha \in \mathcal{S}} \frac{n_\alpha}{N} X_\alpha, \sum_{{\alpha \in \mathcal{S}}} \frac{n_\alpha}{N} Y_\alpha, \sum_{\alpha \in \mathcal{S}} \frac{n_\alpha}{N} Z_\alpha \right).$$
The resulting max-entropy distribution is our best guess for the underlying distribution without any modification for fairness.
(b) $\theta^b$ and (c) $\theta^w,$
as defined in Remark~\ref{rem:other_theta}.

This results in six distributions; we generate a synthetic datasets from each distribution to use in our evaluation.
We compare the statistical rate, representation rate, divergence from empirical distribution and classifier performance of datasets from these distributions, for varying values of parameter $C$.

\subsection{Comparison across priors and expected value vectors}
We first evaluate the dataset generated using max-entropy distributions with different combinations of prior weights and expected value mentioned earlier.
The results for this evaluation are present in Figure~\ref{fig:maxEnt_compare} and Figure~\ref{fig:maxEnt_compare_clf}.

Figure~\ref{fig:maxEnt_compare}a and Figure~\ref{fig:maxEnt_compare_clf}a show that for both COMPAS and Adult datasets, the max-entropy distributions obtained using prior $q_C^{\textrm{w}}$ achieve higher statistical rate than the distributions obtained using  $q_C^{\textrm{d}}$.
However, the KL-divergence of the max-entropy distributions obtained using expected value $\theta^{\textrm{w}} $ or $\theta^{\textrm{b}} $ are higher as well.
As the samples in the raw dataset are unbalanced with respect to gender, the distributions using balanced marginal distributions (i.e., $q_C^w$) are expected to have a larger divergence from the empirical distribution of raw data than the distributions using the expected value of data.
Note that, according to the application, one can aim to achieve high representation rate or high statistical rate or both in the final distribution.
The max-entropy distribution using $q_C^{\textrm{w}}$ and $\theta^{\textrm{d}}$ achieves high statistical rate and low representation rate, while the max-entropy distribution using $q_C^{\textrm{w}}$ and $\theta^{\textrm{b}}$ achieves high statistical rate and high representation rate.

\subsection{Comparison of classifier trained using different max-entropy distribution datasets}
For the decision tree classifier trained on the generated data, we compute the statistical rate using the predictions to evaluate the effects of different training data on the fairness of the classifier.
In addition, we report the classifier accuracy when trained on each output dataset.
The classifier results are presented in Figure~\ref{fig:maxEnt_compare}c,d and Figure~\ref{fig:maxEnt_compare_clf}c,d.

Once again the the max-entropy distributions obtained using prior distribution $q_C^{\textrm{w}}$ achieve better classifier statistical rate than the distributions obtained using $q_C^{\textrm{d}}$.
The accuracy of the classifiers trained on datasets obtained using prior distribution $q_C^{\textrm{w}}$ is slightly lower than the accuracy of the classifiers trained on distributions obtained using sample uniform weights.
However, it is interesting to note that the significant difference in ``accuracy'' of the data all but disappears when passed through the classifier.
Importantly, the accuracy drops sharply as the value of $C$ increases as $C=1$ assigns equal probability mass to all points in the domain and ignores the original samples.
This suggests a $C$ value in the low-to-mid range would likely optimize accuracy and statistical rate simultaneously.

Figure~\ref{fig:maxEnt_compare_other}b,c presents the Gaussian Naive Bayes classifier statistical rate and accuracy, when trained using different max-entropy distributions on the COMPAS dataset.

\subsection{Comparison of representation rate}
Figure~\ref{fig:maxEnt_compare_other}a shows the variation of representation rate.
As expected, distributions obtained using expected value $\theta^{\textrm{b}}$ or $\theta^{\textrm{w}}$ have representation rate close to 1.

\section{Additional empirical results on larger COMPAS dataset} \label{sec:other_experiments}

In this section, we present additional empirical results on the larger version of the COMPAS dataset.
In the small version of the dataset, the features used were  sex, race, age, priors count, and charge degree as features, and uses a binary marker of recidivism within two years as the label.
The age attribute was categorized into three categories, younger than 25, between 25 and 45, and older than 45, and the priors count attribute is categorized in to three categories (no prior crime, between 1 and 3, and more than 3).
Further, we only considered data for convicted criminals labelled as being either White or Black.

The large dataset consists of attributes sex, race, age, juvenile felony count, juvenile misdemeanor count, juvenile other count, months in jail, priors count, decile score, charge degree, violent crime,  violent recidivism, drug related crime, firearm involved, minor involved, road safety hazard, sex offense, fraud and petty crime, with recidivism as the label.
We did not exclude any samples and we did not categorize any attributes.
The original data contains samples from 6 different races whose age ranged from 18 to 96 with at most 40 prior counts, juvenile felony count, juvenile misdemeanor count, and juvenile other count.

We model the domain $\Omega_L$ for this version as $\set{0,1}^{8} \times \set{0,1,2}^3 \times  \set{0,1, \dots 5} \times \Delta_6 \times \set{0,1, \dots 7}^2 \times \set{0,1, \dots 10}^2  \times \set{0,1, \dots 11} \times \set{0,1, \dots 13}$.
Overall the domain contains approximately $1.4 \times 10^{11}$ different points.
\subsection{Evaluating the statistical rate and accuracy of generated dataset}
We evaluate the dataset generated using different max-entropy algorithms.
We run the algorithm with different combinations of prior weights and expected value mentioned earlier.
We vary the $C$ value for our framework and measure the statistical rate of the output distribution.

For this dataset, calculating the KL-divergence from empirical distribution is difficult due to the large domain size. Hence we consider another metric to check how well the max-entropy distribution preserves the pairwise correlation between features.
To calculate this, we first calculate the covariance matrix of the output dataset, say $\text{Cov}_{\text{output}}$ and the original raw dataset  $\text{Cov}_{\text{data}}$, and then report the Frobenius norm of the difference of these matrices, i.e., $\norm{ \text{Cov}_{\text{output}} -  \text{Cov}_{\text{data}}}_F^2$.
The lower the value of the norm, the better the output distribution preserves the pairwise correlation.
The results for this evaluation are present in Figure~\ref{fig:DI_Utlity_MaxEnt_large}.
Here again the first part of the figure shows that the max-entropy distributions obtained using prior $q_C^{\textrm{w}}$ and expected value $\theta^{\textrm{w}} $ or $\theta^{\textrm{b}} $ achieve higher statistical rate values than the distributions obtained from max-entropy distribution obtained using uniform weights on samples.
Similarly the representation rate of max-entropy distributions using prior distribution $q_C^{\textrm{w}}$ and expected value $\theta^{\textrm{w}} $ or $\theta^{\textrm{b}} $ are close to 1.0.

\subsection{Evaluating the statistical rate and accuracy of classifier trained on generated dataset}
As mentioned earlier, we use the generated datasets to train a Gaussian Naive Bayes and the Decision Tree Classifier and evaluate the fairness and the accuracy of the resulting classifier.

Firstly, we again vary the $C$ value for our framework and measure the statistical rate of the output of the classifier as well as the accuracy.
The results for this evaluation using Gaussian Naive Bayes are present in Figure~\ref{fig:Clf_DI_Utlity_MaxEnt_GNB_large} and using Decision Tree Classifier are present in Figure~\ref{fig:Clf_DI_Utlity_MaxEnt_DTC_large}.
As expected, once again the the max-entropy distributions obtained using prior distribution $q_C^{\textrm{w}}$ achieve higher statistical rate values than the distributions obtained from max-entropy distribution obtained using uniform weights on samples.
The accuracy also drops as the value of $C$ tends to 1. This is again because the prior distribution in case of $C=1$ assigns equal probability mass to all points in the domain.

\begin{algorithm}[t]
   \caption{\textbf{Value-Oracle}: Computing dual function value at any point $\lambda$}
   \label{algo:value_oracle}
\begin{algorithmic}[1]
  \STATE {\bfseries Input:} samples $\mathcal{S}:= \set{\alpha_i}_{i\in N}\subseteq \set{0,1}^n$, weights $w \in \Delta_{N-1}$, smoothing parameter $C \in [0,1]$ expected vector $\theta$ and vector $\lambda$
   \STATE $g_1 \gets 1$  
   \FOR{$j \in \set{1, \dots, n}$}
   	\STATE $s_j^0 \gets e^{\lambda_j}/2$
   	\STATE $g_1 \gets g_1 \cdot s_j^0$
   \ENDFOR
   \STATE $g_2 \gets 0$  
   \FOR{$i \in \set{1, \dots, N}$}
   	\STATE $g_2 \gets g_2 + w_i \cdot e^{\inner{\alpha_i, \lambda}}$
   \ENDFOR   
   \STATE $g \gets Cg_1 + (1-C)g_2$
   \STATE return $\log (g) - \inner{\theta, \lambda}$
\end{algorithmic}
\end{algorithm}

\begin{algorithm}[t]
   \caption{\textbf{Gradient-Oracle:} Computing gradient of dual function at any point $\lambda$}
   \label{algo:gradient_oracle}
\begin{algorithmic}[1]
  \STATE {\bfseries Input:} samples $\mathcal{S}:= \set{\alpha_i}_{i\in N}\subseteq \set{0,1}^n$, weights $w \in \Delta_{N-1}$, smoothing parameter $C \in [0,1]$ expected vector $\theta$ and vector $\lambda$
   \STATE $g_1 \gets 0$  
   \FOR{$j \in \set{1, \dots, n}$}
   	\STATE $s_j^0 \gets e^{\lambda_j}/2$
   	\STATE $s_j^1 \gets e_j \cdot e^{\lambda_j}/2$ \COMMENT{$e_j$ is standard basis vector with 1 in $j$-th location}
   \ENDFOR
   \FOR{$j \in \set{1, \dots, n}$}
   	   \STATE $t \gets 1$
	   \FOR{$k \in \set{1, \dots, n} \setminus \set{j}$}
		   		\STATE $t \gets t \cdot s_k^0$
	   \ENDFOR
	   \STATE $g_1 \gets g_1 + s_j^1 \cdot t$
   \ENDFOR
   
   \STATE $g_2 \gets 0$  
   \FOR{$i \in \set{1, \dots, N}$}
   	\STATE $g_2 \gets g_2 + \alpha \cdot w_i \cdot e^{\inner{\alpha_i, \lambda}}$
   \ENDFOR   
   \STATE $g \gets Cg_1 + (1-C)g_2$
   \STATE $v \gets $Value-Oracle $(S, w, C, \theta, \lambda) + \inner{\theta, \lambda}$ 
   \STATE $v_2 \gets e^{v}$
   \STATE return $\frac{1}{v_2} g - \theta$
\end{algorithmic}
\end{algorithm}

\begin{algorithm}[t]
   \caption{\textbf{Hessian-Oracle:} Computing hessian of dual function at any point $\lambda$}
   \label{algo:hessian_oracle}
\begin{algorithmic}[1]
  \STATE {\bfseries Input:} samples $\mathcal{S}:= \set{\alpha_i}_{i\in N}\subseteq \set{0,1}^n$, weights $w \in \Delta_{N-1}$, smoothing parameter $C \in [0,1]$ expected vector $\theta$ and vector $\lambda$
   \STATE $g_1 \gets 0$  
   \FOR{$j \in \set{1, \dots, n}$}
   	\STATE $s_j^0 \gets (e^{(1-\theta_j)\lambda_j} + e^{-\theta_j\lambda_j})/2$
   	\STATE $s_j^1 \gets e_j \cdot (e^{(1-\theta_j)\lambda_j})/2$ \COMMENT{$e_j$ is standard basis vector with 1 in $j$-th location}
   	\STATE $s_j^2 \gets  e_j e_j^\top \cdot (e^{(1-\theta_j)\lambda_j})/2$
   \ENDFOR
   \FOR{$j \in \set{1, \dots, n}$}
   	   \STATE $t_1 \gets 1$
   	   \STATE $t_2 \gets 0$
	   \FOR{$k \in \set{1, \dots, n} \setminus \set{j}$}
		   		\STATE $t_1 \gets t_1 \cdot s_k^0$
		   	   \STATE $t_3 \gets 1$
			   \FOR{$l \in \set{1, \dots, n} \setminus \set{j, k}$}
				   		\STATE $t_3 \gets t_3 \cdot s_l^0$
		   		\ENDFOR		 
		   		\STATE $t_2 \gets t_2 + s_j^1 {s_k^1}^\top \cdot t_3$  		
	   \ENDFOR
	   \STATE $g_1 \gets g_1 + s_i^2 \cdot t_1 + t_2$
   \ENDFOR
   
   \STATE $g_2 \gets 0$  
   \FOR{$i \in \set{1, \dots, N}$}
   	\STATE $g_2 \gets g_2 + \alpha \alpha^\top \cdot w_i \cdot e^{\inner{\alpha_i - \theta, \lambda}}$
   \ENDFOR   
   \STATE $g \gets Cg_1 + (1-C)g_2$
   \STATE $v_1 \gets $Value-Oracle $(S, w, C, \theta, \lambda)$ 
   \STATE $v_2 \gets $Gradient-Oracle $(S, w, C, \theta, \lambda)$ 
   \STATE $v_3 \gets \frac{1}{v_1} g - (v_2+\theta)(v_2 - \theta)^\top$
   \STATE return $v_3$
\end{algorithmic}
\end{algorithm}

\section{Full algorithm for max-entropy optimization} \label{sec:full_algorithm}
In this section, we state the full-algorithm for max-entropy optimization.
The algorithm is based on the second-order framework of \cite{ALOW17,CMTV17}.
We start with a complete algorithm for value, gradient and Hessian oracles for $h_{\theta, q_C^w}$, constructed along similar lines as the proof of Lemma~\ref{lem:counting-oracles}.

\subsection{Oracle algorithm}
Algorithm~\ref{algo:value_oracle} shows how to compute  the dual function $h_{\theta, q_C^w}$ value at any point $\lambda$,
Algorithm~\ref{algo:gradient_oracle} shows how to compute  the gradient of the dual function at any point $\lambda$,
and 
Algorithm~\ref{algo:gradient_oracle} shows how to compute  the Hessian of the dual function at any point $\lambda$.

\subsection{Max-entropy optimization algorithm}

With the first and second order oracles, we can now state our entire algorithm for the hypercube domain.
Algorithm~\ref{algo:final_algorithm} presents the approach to optimizing the dual of the max-entropy program.
The inner optimization problem \eqref{inner-Opt} is a quadratic optimization problem and can be solved in polynomial time using standard interior-point methods \cite{karmarkar1984new,wright2005interior}.

\begin{algorithm}[t]
   \caption{ Full algorithm to compute  max-entropy distributions}
   \label{algo:final_algorithm}
\begin{algorithmic}[1]
  \STATE {\bfseries Input:} samples $\mathcal{S}:= \set{(X_i, Y_i, Z_i)}_{i\in N}\subseteq \set{0,1}^n$, parameter $C \in [0,1]$, target expected value $\theta$, weights $\set{w_i}_{i=1}^N \in \Delta_{N-1}$ and   $\varepsilon > 0$
	\STATE $q_C^w \gets $ Prior distribution constructed using $\set{w_i}_{i=1}^N$ and $C$
	\STATE $R \gets 8n\log \nicefrac{1}{C\varepsilon}$ 
	\STATE $T \gets 16nR\log \nicefrac{1}{C\varepsilon}$ 
	\STATE $\lambda \gets \textbf{0}$
	\FOR{$i=1$ {\bfseries to} $T$}
		\STATE $g \gets $ Gradient-Oracle $(S, w, C, \theta, \lambda)$
		\STATE $H \gets $ Hessian-Oracle $(S, w, C, \theta, \lambda)$
   		\STATE $y_\varepsilon \gets$  $\frac{\varepsilon}{8n R}$-approximate minimizer of the following convex quadratic program (using primal path following algorithm \cite{karmarkar1984new,wright2005interior}),
  		\begin{align*}
		  	\inf_{y\in\R^n}  \inner{g , y} + \frac{1}{2e} y^\top H y \\
		  	\text{s.t. }\norm{y}_\infty\leq \frac{1}{8n} \; \; \mbox{ and }
			\norm{\lambda+y}_\infty\leq R \eqlabel{inner-Opt}
		  \end{align*}
	   \STATE $\lambda \gets \lambda + \nicefrac{y_\varepsilon}{e^2}$ 
   \ENDFOR
	\STATE return $\lambda$
\end{algorithmic}
\end{algorithm}

\subsection{Time complexity of Algorithm~\ref{algo:final_algorithm}}
To provide a time complexity bound for Algorithm~\ref{algo:final_algorithm}, we will invoke the bounds proved by \cite{CMTV17} for optimization of \textit{second-order robust} functions. 
\begin{theorem}[\textbf{Run time of the Box constrained Newton's method,~\cite{ALOW17}}]
Given access to the first and second order oracles for $\alpha$-second order robust function $f:\R^n\to\R$, $\varepsilon>0$, promise of $\ell_\infty$ ball of radius $R_\varepsilon$ containing $\varepsilon$-approximate minimizer of $f$, starting point $x\in\R^n$ with $\norm{x}_\infty \leq R_\varepsilon$, Algorithm~\ref{algo:final_algorithm} runs for $O\left(\alpha R_\varepsilon \log\left(\frac{\mathrm{var}_{R_\varepsilon}(f)}{\varepsilon}\right)\right)$ iterations and outputs $3\varepsilon$-approximate minimizer of $f$ where $\mathrm{var}_{R_\varepsilon}(f):=\max_{x,y \mid \norm{x}_1,\norm{y}_1\leq R_\varepsilon} f(x) - f(y)$.
\label{thm:bcnm-runtime}
\end{theorem}

\noindent
In particular, for our max-entropy framework, this algorithm runs in time polynomial in $d$, $N$ and the bit complexity of the input parameters, provided
\begin{enumerate}
    \item there is a bound on the size of dual solution, $\lambda^\star$,
    \item efficient first and second-order oracles for the dual function,
    \item the dual function is \textit{second-order robust}.
\end{enumerate}
We have already shown that $\norm{\lambda^\star}$ is bounded (Lemma~\ref{lem:bounding_box}) as well as provided fast first and second-order oracles (Lemma~\ref{lem:counting-oracles}).
To establish to polynomial time complexity of this algorithm, we just need to prove that dual function is second-order robust.
A convex function $f:\R^n\to\R$ is said to be $\alpha$-second order robust, if for all $x,y\in\R^n$ with $\norm{y}_\infty\leq 1$ satisfies
\[\abs*{D^3 f(x)[y,y,y]}\leq \alpha D^2 f(x)[y,y]
\]
where $D^k f(x)[y,\ldots,y]:=\left.\frac{d^k}{dt^k}f(x+ty)\right\rvert_{t=0}$.
The following lemma establishes the second-order robustness of the dual function $h_{\theta, q}$.
\begin{lemma}[\textbf{Second-order robustness of the dual-MaxEnt function}]
Given  $\Omega = \set{0,1}^n$, prior $q:\Omega\to[0,1]$ and the target expected vector $\theta\in\conv(\Omega)$, the dual maximum entropy function $h_{\theta,q}(\lambda):=\log\left(\sum_{\alpha\in\Omega} q(\alpha) e^{\inner{\lambda,\alpha-\theta}}\right)$ is $4n$-second order robust.
\label{lem:second-order-robustness}
\end{lemma}
\noindent
Using this second-order robustness property, bound on $\norm{\lambda^\star}$, gradient, Hessian oracles and interior point method to solve the inner-optimization problem \eqref{inner-Opt}, as a corollary of Theorem 3.4 in \cite{CMTV17}, it follows that Algorithm~\ref{algo:final_algorithm} runs in time polynomial in $d$, $N$ and bit complexities of all the numbers involved.

Before proving the lemma, we state and prove the following general claim in the proof.
\begin{claim} \label{clm:expected_bound}
Let $X$ be a real valued random variable over the discrete set $\Omega$ with $\abs{X}\leq r$ for some constant $r\in\R_+$.
Then,
\[
\abs{\E{X^3}-\E{X^2}\E{X}}\leq 2r(\E{X^2}-\E{X}^2).
\]
\end{claim}
\begin{proof}
Let us denote the probability mass function of $X$ with $p$.
Then,
\begin{align*}
&\E{X^3}-\E{X^2}\E{X}
= \sum_{\alpha\in\Omega} X(\alpha)^3 p(\alpha) - \sum_{\alpha,\beta\in\Omega} X(\alpha)^2X(\beta) p(\alpha)p(\beta)\\
&= \frac{1}{2}\sum_{\alpha,\beta\in\Omega}(X(\alpha)^3 - X(\alpha)^2 X(\beta))p(\alpha)p(\beta) + \frac{1}{2}\sum_{\alpha,\beta\in\Omega}(X(\beta)^3 - X(\alpha)X(\beta)^2 )p(\alpha)p(\beta)\\
& = \frac{1}{2}\sum_{\alpha,\beta\in\Omega}(X(\alpha)-X(\beta))^2(X(\alpha)+X(\beta)) p(\alpha)p(\beta).
\end{align*}
We also note that, $\abs{X(\alpha)+X(\beta)}\leq 2r$ for any $\alpha,\beta\in\Omega$ as $\abs{X}\leq r$.
Therefore,
\begin{align*}
\abs{\E{X^3}-\E{X^2}\E{X}}
 &= \frac{1}{2}\abs*{\sum_{\alpha,\beta\in\Omega}(X(\alpha)-X(\beta))^2(X(\alpha)+X(\beta)) p(\alpha)p(\beta)}\\
&\qquad \leq r\sum_{\alpha,\beta\in\Omega}(X(\alpha)-X(\beta))^2p(\alpha)p(\beta)\\
&\qquad = 2r(\E{X^2}-\E{X}^2).
\end{align*}
\end{proof}

\begin{proof}[Proof of Lemma~\ref{lem:second-order-robustness}]
Let us fix a point $\lambda_0\in\R^n$ and a direction $\lambda_1\in\R^n$ with $\norm{\lambda_1}_\infty\leq 1$.
We need to verify that 
\begin{equation}
\abs*{D^3 h_{\theta,q}(\lambda_0)[\lambda_1,\lambda_1,\lambda_1]}\leq 4 n D^2 h_{\theta,q}(\lambda_0)[\lambda_1,\lambda_1]
\label{eq:robustness-part0}
\end{equation}
 to show that $h_{\theta,q}$ is $4n$-second order robust.
 
For any $k \in \mathbb{Z}$, let $g_q^{(k)}$ denote the following function. 
\begin{equation*}
    g_q^{(k)}(\lambda_0, \lambda_1) = \sum_{\alpha \in \Omega} q(\alpha) \cdot \inner{\lambda_1, \alpha}^k \cdot e^{\inner{\lambda_0, \alpha}},
\end{equation*}
Then the derivative $D^2 h_{\theta,q}(\lambda_0)[\lambda_1,\lambda_1]$ can be written as 
\begin{equation*}
    D^2 h_{\theta,q}(\lambda_0)[\lambda_1,\lambda_1] = \frac{g_q^{(2)}(\lambda_0, \lambda_1)}{g_q^{(0)}(\lambda_0, \lambda_1)} - \frac{g_q^{(1)}(\lambda_0, \lambda_1)^2}{g_q^{(0)}(\lambda_0, \lambda_1)^2}
\end{equation*}
Similarly,
\begin{align*}
    D^3 h_{\theta,q}(\lambda_0)[\lambda_1,\lambda_1,\lambda_1] &= \frac{g_q^{(3)}(\lambda_0, \lambda_1)}{g_q^{(0)}(\lambda_0, \lambda_1)} + \frac{2g_q^{(1)}(\lambda_0, \lambda_1)^3}{g_q^{(0)}(\lambda_0, \lambda_1)^3} 
    - \frac{3g_q^{(2)}(\lambda_0, \lambda_1) g_q^{(1)}(\lambda_0, \lambda_1)}{g_q^{(0)}(\lambda_0, \lambda_1)^2}
\end{align*}
We begin by dividing $D^3 h_{\theta,q}(\lambda_0)[\lambda_1,\lambda_1,\lambda_1]$ into two parts, and prove upper bounds on each part individually.
Firstly note that using Cauchy-Swartz, we can bound $g_q^{(1)}$ using $g_q^{(0)}$ in the following way,
\begin{align*}
    g_q^{(1)}(\lambda_0, \lambda_1) &= \sum_{\alpha \in \Omega} q(\alpha) \cdot \inner{\lambda_1, \alpha} \cdot e^{\inner{\lambda_0, \alpha}} \\
    &\leq \sum_{\alpha \in \Omega} q(\alpha) \cdot \norm{\lambda_1}_\infty \norm{\alpha}_1 \cdot e^{\inner{\lambda_0, \alpha}} \\
    &\leq \max_{\alpha \in \Omega} \norm{\alpha}_1 \cdot g_q^{(0)}(\lambda_0, \lambda_1) \\
    &\leq  n\cdot g_q^{(0)}(\lambda_0, \lambda_1)
\end{align*}
since $\norm{\lambda_1}_\infty\leq 1$ and $\max_{\alpha \in \Omega} \norm{\alpha}_1 \leq n$, as all features in $\Omega$ are binary.
Now using this property, we get that 
\begin{align}
    \left| \frac{2g_q^{(1)}(\lambda_0,\lambda_1)^3}{g_q^{(0)}(\lambda_0, \lambda_1)^3}
    - \frac{2g_q^{(2)}(\lambda_0, \lambda_1) g_q^{(1)}(\lambda_0, \lambda_1)}{g_q^{(0)}(\lambda_0, \lambda_1)^2} \right| \nonumber 
    &=  \left| \frac{2g_q^{(1)}(\lambda_0,\lambda_1)}{g_q^{(0)}(\lambda_0, \lambda_1)} \right| \cdot  D^2 h_{\theta,q}(\lambda_0)[\lambda_1,\lambda_1] \nonumber \\
    &\leq 2n \cdot  D^2 h_{\theta,q}(\lambda_0)[\lambda_1,\lambda_1]. \label{eq:robust_1}
\end{align}
Next we try to bound the second part of $D^3 h_{\theta,q}(\lambda_0)[\lambda_1,\lambda_1,\lambda_1]$. To do so, let $p_{\lambda_0} : \Omega \to [0,1]$ denote the following distribution
\begin{equation*}
    p_{\lambda_0}(\alpha) = \frac{q(\alpha) e^{\inner{\lambda_0, \alpha}}}{g_q^{(0)}(\lambda_0, \lambda_1)}.
\end{equation*}
Then using Claim~\ref{clm:expected_bound} and the fact $\max_{\alpha \in \Omega} \norm{\alpha}_1 \leq n$, we get
\begin{align}
     \left| \frac{g_q^{(3)}(\lambda_0, \lambda_1)}{g_q^{(0)}(\lambda_0, \lambda_1)} - \frac{g_q^{(2)}(\lambda_0, \lambda_1) g_q^{(1)}(\lambda_0, \lambda_1)}{g_q^{(0)}(\lambda_0, \lambda_1)^2} \right| \nonumber 
    &=  \left|  \mathbb{E}_{p_{\lambda_0}}[\inner{\lambda_1, \alpha}^3] - \mathbb{E}_{p_{\lambda_0}}[\inner{\lambda_1, \alpha}^2] \mathbb{E}_{p_{\lambda_0}}[\inner{\lambda_1, \alpha}] \right| \nonumber \\
    &\leq 2n \left( \mathbb{E}_{p_{\lambda_0}}[\inner{\lambda_1, \alpha}^2] - \mathbb{E}_{p_{\lambda_0}}[\inner{\lambda_1, \alpha}]^2 \right) \nonumber \\
    &= 2n \cdot  D^2 h_{\theta,q}(\lambda_0)[\lambda_1,\lambda_1]. \label{eq:robust_2}
\end{align}
Combining \ref{eq:robust_1} and \ref{eq:robust_2} using the triangle inequality, we get that
\begin{equation*}
\abs*{D^3 h_{\theta,q}(\lambda_0)[\lambda_1,\lambda_1,\lambda_1]}\leq 4 n D^2 h_{\theta,q}(\lambda_0)[\lambda_1,\lambda_1].
\end{equation*}
Therefore, $h_{\theta,q}$ is $4n$-second order robust.
\end{proof}

\end{document}